%% file: main.tex
\documentclass[11pt]{article}
\usepackage{fullpage}
\usepackage{url}
\usepackage{xcolor}

\usepackage{graphicx}
\usepackage{subcaption}
\usepackage[numbers]{natbib}
\usepackage{amssymb}
\usepackage{amsmath}
\usepackage{amsthm}
\usepackage{microtype}
\usepackage{booktabs}

\usepackage[backref]{hyperref}
\hypersetup{
    colorlinks=true,
    urlcolor=blue,
    linkcolor=blue,
    citecolor=blue,
    linktocpage=true
}
\usepackage{cleveref}

\newcommand{\E}{\mathbb E}
\newcommand{\cD}{\mathcal D}
\newcommand{\cJ}{\mathcal J}
\newcommand{\ind}{\mathbb I}
\newcommand{\R}{\mathbb R}
\newcommand{\CAL}{\mathsf{CAL}}
\newcommand{\cdl}{\mathsf{CDL}}
\newcommand{\scdl}{\mathsf{SCDL}}
\newcommand{\scfdl}{\mathsf{SCFDL}}
\newcommand{\ECE}{\mathsf{ECE}}
\newcommand{\distcal}{\mathsf{DistCal}}
\newcommand{\smCE}{\mathsf{smCE}}
\newcommand{\Ber}{\mathsf{Ber}}
\newcommand{\sr}{\mathsf{SR}}
\newcommand{\cutoff}{\mathsf{Cutoff}}
\newcommand{\bECE}{\mathsf{bECE}}

\DeclareMathOperator*{\argmax}{arg\,max}

\newtheorem{theorem}{Theorem}[section]
\newtheorem{lemma}[theorem]{Lemma}
\newtheorem{claim}[theorem]{Claim}
\newtheorem{definition}{Definition}[section]
\newtheorem{remark}{Remark}
\newtheorem{example}{Example}
\newtheorem{question}{Question}

\renewcommand{\citet}{\cite}

\title{Testable and Actionable Calibration for Full Swap Regret 
}
\author{Konstantina Bairaktari\thanks{Department of Computer Science, Aarhus University. Supported by the European Union (ERC, TUCLA, 101125203). Views and opinions expressed are however those of the author(s) only and do not necessarily reflect those of the European Union or the European Research Council. Neither the European Union nor the granting authority can be held responsible for them. Supported by NSF awards CNS-2232692 and CCF-2311649, while at Northeastern University. } \and Lunjia Hu\thanks{Khoury College of Computer Sciences, Northeastern University.} \and Huy L. Nguyen\thanks{Khoury College of Computer Sciences, Northeastern University. Supported by NSF award CCF-2311649.} \and Jonathan Ullman\thanks{Khoury College of Computer Sciences, Northeastern University.  Supported by NSF awards CNS-2232692 and CNS-2247484.} }
\date{}

\allowdisplaybreaks
\begin{document}
\maketitle
\input{body}

\bibliographystyle{alphaurl}
\bibliography{refs}
\appendix
\input{appendix}
\end{document}

%% file: body.tex
\thispagestyle{empty}
\begin{abstract}
AI generated predictions increasingly inform decision making in critical tasks, and therefore must be trustworthy.  One widely used measure of trustworthiness is calibration, which requires that the predictions match the true frequencies and can be treated like real probabilities of a given outcome.  However, defining calibration is subtle, and designing good measures of calibration error has been an active topic of recent research.  The first goal is to find calibration measures that are \emph{actionable}, meaning they can inform decision makers about their utility loss when predictions are treated as true probabilities, which is known as swap regret.  The second goal is to find calibration measures that are \emph{testable}, meaning that calibration error can be measured from a small sample of predictions and outcomes.  Although these are very basic requirements, there is no existing calibration measure that fully satisfies both properties, and all existing measures relax actionability by bounding a weaker notion of swap regret, or relax testability by having suboptimal estimation error. We introduce a new calibration measure, Soft-Binned Calibration Decision Loss ($\scdl$), which we prove is fully actionable without weakening either requirement, and testable with nearly optimal error rate. In addition, $\scdl$ satisfies other desired properties such as continuity and consistency.  We also provide a set of experiments confirming that the theoretical advantages of $\scdl$ compared to other measures lead to better performance in practice.
\end{abstract}

\newpage

\section{Introduction}
\label{sec:intro}

Deciding whether to carry an umbrella in the morning seems like a trivial problem, but it exemplifies the fundamental challenge of decision making under uncertainty, where we have to commit to an action without knowing the outcome in advance.  To navigate this uncertainty, we turn to forecasts, a predicted probability of rain that compresses complex atmospheric dynamics into a single actionable number.  But when can we trust this number?  A forecast that reports a $20\%$ chance of rain is meaningful if that number faithfully reflects the true probability of rain, but if it systematically overstates or understates the true probability, then decisions based on this forecast will be systematically wrong. The stakes grow considerably when we move from umbrellas to medical diagnoses where acting on an unreliable prediction can have severe consequences.

A decision maker can trust a predicted probability when it can be interpreted as a true probability and used to make the best possible decisions.  Intuitively, when a forecast reports a $20\%$ chance of rain, it should rain on $20\%$ of the days that receive that forecast. This property, known as (perfect) calibration, has its roots in the classical forecasting literature \citep{calibration}. Formally, given a predictive model that outputs probabilities
$p \in [0,1]$ for binary outcomes $y \in \{0, 1\}$, perfect calibration requires that the joint distribution $\cD$ of $(p,y)$ satisfies
\[
\E_{(p,y)\sim \cD}[y\mid p] = p.
\]

Perfect calibration is an ideal, it is never really achieved in practice, so the next question we need to ask is how far a predictor is from perfectly calibrated.  Quantifying this calibration error turns out to be less straightforward than it may appear.  A classic calibration measure is the \emph{Expected Calibration Error ($\ECE$)}, defined as the expected absolute difference between the prediction $p$ and the true probability $\E[y | p]$, where the expectation is taken over the predictions made by the model.  While the definition is very natural, $\ECE$ has the problem of being discontinuous and sensitive to small perturbations in the predictions, limiting its applicability for a robust and meaningful measure of calibration error. Many recent works aim to design better calibration measures that satisfy a variety of desired properties \cite{metrics-cal,utc,smoothECE,multiclass,cdl,truthful-HQYZ,test-action,truthful-QZ,truthful,efficient-decision,importance}.

In this work, we build on the line of research that evaluates calibration through the lens of downstream decision making, where we evaluate calibration not based on the differences between the predicted and true probabilities, but by the impact these differences have on decision makers who use the predicted probabilities.  In the literature, measures of calibration error that provide guarantees about the impact of miscalibration on the quality of downstream decisions are called \emph{actionable} \cite{test-action}.

Measures of calibration error are a property of a predictor and a data distribution.  But to be useful, the measure of calibration error must also be \emph{testable}, meaning we can accurately estimate the calibration error using a small finite sample of predictions and outcomes.  Without testability, we do not know which predictors actually perform well or how to improve our predictors to be more calibrated.

In this work, our goal is to design calibration measures that are simultaneously \emph{fully actionable}, meaning they bound the impact of miscalibration on arbitrary downstream decision makers, and \emph{nearly optimally testable}, meaning we can approximate the calibration error from samples with the best possible rate of convergence.  Our main contribution is to introduce the first measure of calibration error that satisfies both properties. We emphasize that our focus is on the fundamental question of measuring the calibration error of a given predictor, rather than on learning a calibrated predictor. 

\subsection{An actionable calibration measure: Calibration Decision Loss}

The quality of decisions can be evaluated through a utility function $U(a,y)$ that assigns a value to each combination of action $a$ (take an umbrella or don't) and outcome $y$ (rainy or not). A perfectly calibrated predictor allows the decision maker to act on the predicted probabilities as if they were true without suffering any \emph{regret}, which is the feeling that they should have interpreted the predictions differently in hindsight, and swapped certain actions for others.

\begin{example}
Consider a distribution $\cD$ over (prediction, outcome), where the prediction is the probability of rain on some day and the outcome is $1$ if it rains and $0$ otherwise. The predictor says there is a $20\%$ chance of rain with probability $\frac12$ and says there is an $80\%$ chance of rain with probability $\frac12$. If the predictor is calibrated, then it rains on $20\%$ of the days where the forecast is $20\%$ and on $80\%$ of the days with forecast $80\%$. Now suppose the decision maker carries an umbrella whenever the predicted probability exceeds $50\%$, and receives utility $1$ for a correct match (carrying an umbrella on a rainy day or not carrying one on a dry day) and $0$ otherwise. Under this strategy, the expected utility is $\frac12 \times 0.8 + \frac12 \times 0.8 = 0.80$. Since the decision maker has no information to go on other than the prediction, and chooses an action based solely on the predicted probability, a swap would mean choosing a different action for a given prediction value. In this case, no such swap improves the expected utility: carrying an umbrella when $p=0.8$ and not carrying one when $p=0.2$ is already optimal in hindsight.
\end{example}

The connection between perfect calibration and decision making can be made precise: the joint distribution over predictions and outcomes is (perfectly) calibrated if and only if no decision maker incurs \emph{swap regret} when relying solely on the predictions to take actions \citep{foster1997calibrated}. Let $\cD$ be a joint distribution over predictions and outcomes $(p,y)\in [0,1]\times \{0,1\}$. For a decision task $Z = (A,U)$ with action space $A$ and utility function $U$, the swap regret of a response function $r:[0,1]\to A$ is defined as 
\begin{equation}
\label{eq:intro-swap}
\sr_Z(r,\cD):= \sup_{\sigma:A\to A}\E_{(p,y)\sim \cD}[U(\sigma(r(p)),y) - U(r(p),y)].
\end{equation}
Upon receiving a prediction $p$, a decision maker who trusts it as the true probability will respond with the action $a\in A$ that maximizes the expected utility \emph{assuming} that the outcome $y$ is drawn from the Bernoulli distribution with mean $p$. Formally, the \emph{best-response} function $r_Z^*$ for a decision task $Z = (A,U)$ is defined as
\[
r_Z^*(p):= \argmax_{a\in A}\E_{y\sim \Ber(p)}U(a,y).
\]
A distribution $\cD$ of predictions and outcomes is calibrated if and only if for every decision task $Z$, the best-response function incurs no swap regret: $\sr_Z(r_Z^*,\cD) = 0$. From the perspective of decision makers, this property defines what it means to be calibrated.

Thus, it is natural to measure the calibration error of a distribution $\cD$ using the swap regret incurred for decision tasks. Hu and Wu \cite{cdl} introduced a calibration error measure called the Calibration Decision Loss ($\cdl$) as the swap regret $\sr_Z(r_Z^*,\cD)$ maximized over all decision tasks $Z$ with a bounded utility function $U:A\times \{0,1\}\to [0,1]$. A distribution $\cD$ is calibrated if and only if $\cdl(\cD) = 0$. A low $\cdl$ provides the meaningful guarantee of low swap regret for any downstream task $Z$ when the decision maker simply uses the best-response function $r_Z^*$.%
Because of this property, Rossellini et al.\ \cite{test-action} describe $\cdl$ as an \emph{actionable} measure of calibration.

\subsection{The challenge of being both actionable and testable}

Although $\cdl$ is actionable, it is fundamentally not \emph{testable} \cite{test-action}, and cannot be accurately estimated from a finite dataset of predictions and outcomes. The intuition is that even a perfectly calibrated predictor can exhibit high $\cdl$ on a finite sample, so that simply measuring $\cdl$ on a finite sample cannot distinguish a perfectly calibrated predictor from a badly miscalibrated one.  To be precise, when evaluating $\cdl$ on a sample $S$, we treat $S$ as the uniform distribution over the data points in the sample.
The following example shows that $\cdl$ on a finite sample can be bounded away from zero even when the underlying distribution is perfectly calibrated, for any sample size $T$.

\begin{example}
\label{example:2}
Consider the decision task $Z = (A,U)$ where $A = [0,1]$ and $U(a,y) = 1 - (a - y)^2$. The best-response function $r_Z^*$ is the identity function $r_Z^*(p) = p$, so the swap regret for this decision task can be simplified as
\[
\sr_Z(r_Z^*,\cD) = \sup_{\sigma:[0,1]\to [0,1]}\E[(p - y)^2 - (\sigma(p) - y)^2].
\]
Let $\cD$ be the calibrated distribution of $(p,y)$ where $p$ is drawn uniformly from $[1/3, 2/3]$ and $y$ is drawn from $\Ber(p)$. On a finite sample $S$ of $T$ i.i.d.\ points $(p_1,y_1),\ldots,(p_T,y_T)\sim \cD$, the predictions $p_1,\ldots,p_T$ are almost surely all distinct, and thus a swap function $\sigma$ can map each $p_t\in [1/3,2/3]$ to the corresponding $y_t\in \{0,1\}$, incurring $\Omega(1)$ swap regret:
\[
\sr_Z(r_Z^*,S) = \sup_{\sigma:[0,1]\to [0,1]}\frac 1 T\sum_{t = 1}^T \left((p_t - y_t)^2 - (\sigma(p_t) - y_t)^2\right) = \frac 1 T\sum_{t = 1}^T (p_t - y_t)^2 \ge 1/9.
\]
Since $\cdl(S)$ is an upper bound on the swap regret $\sr_Z(r_Z^*,S)$, we also have $\cdl(S) \ge 1/9$.
\end{example}

\Cref{example:2} shows that a calibration error that is meaningful on the full distribution $\cD$ may not be very meaningful on a finite sample $S$, no matter the sample size.  In practice the predictor performance is almost always evaluated on finite data, and we need to be able to estimate the calibration error on new samples from the distribution based on the calibration error we observe on the sample. 

We say a calibration measure $\CAL$ is \emph{testable} if it incurs a small estimation error $|\CAL(\cD) - \CAL(S)|$ between the full distribution $\cD$ and a finite sample $S$.\footnote{This requirement is stronger than the testability required by Rossellini et al.\ \cite{test-action}, since they allow an arbitrary estimator to be used on the sample $S$ in order to estimate $\CAL(\cD)$, whereas we require the calibration measure itself to be an accurate plug-in estimator. For the calibration measures we discuss in this paper, the two notions give the same bounds.}  Unfortunately, \Cref{example:2} shows that being testable is incompatible with providing the actionable no-swap-regret guarantee for the best-response function, even when we only require testability on perfectly calibrated distributions $\cD$ in the limit of infinite sample size.  Formally, \Cref{example:2} implies the following claim:
\begin{claim}
\label{claim:impossibility}
No calibration measure $\CAL$ can satisfy the following two properties at the same time:
\begin{enumerate}
    \item \textbf{Actionable for the best-response function.} For every distribution $\cD$ of $(p,y)\in [0,1]\times \{0,1\}$ and every decision task $Z = (A,U)$ with bounded utility function $U:A\times \{0,1\}\to [0,1]$, we require $\sr_Z(r_Z^*,\cD) \le \CAL(\cD)$.
    \item \textbf{Asymptotically testable for calibrated distributions.} Let $\cD$ be a \emph{calibrated} distribution of $(p,y)\in [0,1]\times \{0,1\}$. Let $S_T$ be a sample of $T$ i.i.d.\ points from $\cD$. We require $\lim_{T\to +\infty}\E[\CAL(S_T)] = 0$.
\end{enumerate}
\end{claim}
Thus, in order to achieve both actionability and testability, we will need to somehow relax one of the two properties.  Specifically, we introduce and motivate a slight variant of actionability that allows us to circumvent this negative result, while preserving the spirit of actionability.  In \Cref{sec:approach} we motivate our relaxation and compare it to other approaches to relaxing actionability.

\paragraph{Approaches to bypassing \Cref{claim:impossibility}}
To overcome the impossibility from \Cref{claim:impossibility}, recent work shows that actionable and testable calibration measures can exist if we weaken actionability by considering only a restricted family of swap functions. Instead of taking the supremum over all swap functions $\sigma:A\to A$ as in \eqref{eq:intro-swap}, Rossellini, Soloff, Barber, Ren, and Willett \cite{test-action} focus on swap functions that correspond to a monotone transformation of the predictions.  They introduce the \emph{cutoff calibration error} ($\cutoff$) and show that it is actionable for this restricted family of swap functions while being testable with a vanishing estimation error $O(T^{-1/2})$. 
Gopalan, Stavropoulos, Talwar, and Tankala \cite{efficient-decision} study calibration measures that are actionable w.r.t.\ a general family  $\mathcal K$ of swap functions and show that the estimation error is characterized by the VC dimension of $\mathsf{thr}(\mathcal K)$, the concept class consisting of thresholds applied to $\mathcal K$. If we were to include all swap functions in $\mathcal{K}$, then we would have the strict definition of actionability for the best-response function, but then the VC dimension of would be infinite and the measure would not be testable.

However, restricting the family of swap functions weakens the decision-theoretic meaning of calibration and introduces arbitrariness in the choice of the swap functions. Intuitively, considering only monotone swaps means we only measure miscalibration from a systemic upward or downward bias in the probabilities, and do not consider the possibility that the predictor may be miscalibrated in other ways such as systemic under or overconfidence.%

\subsection{Our approach: specifying the response function} \label{sec:approach}

Our goal is to design a testable and actionable calibration measure while preserving the strong decision-theoretic guarantee that we should be able to trust calibrated predictions and use them to make good decisions.  The key idea, 
inspired by previous work such as \cite{smoothECE,smooth-decision}, is to change the response function in the definition of actionability to one that is more robust than the exact best-response function $r_Z^*$, and measure the swap regret with respect to that response function.  For example, we use (a variant of) the response function that discretizes the predictions to some appropriate level of precision and then treats the discretized predictions as if they were real probabilities and applies best response.  Intuitively, the goal of the strict definition of actionability is to ensure that the predictions can be used as if they were real probabilities to get strong decision-theoretic guarantees.  By changing the response function to discretize the probabilities, we are instead saying that the high order bits of the predictions can be treated as if they were real probabilities and used to get strong decision-theoretic guarantees.  Thus, this relaxation preserves the spirit of actionability by still giving a bound on the swap regret provided that the predictions are used in an appropriate way.  We note that unlike restricting the class of swap functions, which makes difficult-to-verify assumptions about the predictor and its potential error, changing the response function only requires the user of the predictions to adjust their behavior in a way that is entirely under their own control.

We can see the value of this approach by revisiting \Cref{example:2}. The swap regret in \Cref{example:2} is high because every prediction $p_1,\ldots,p_T$ is distinct, and can be mapped to a distinct action by the identity function $r_Z^*(p) = p$, and thus the swap regret is high because the decision-maker will only see one example of each predicted probability, and in hindsight might have chosen any sequence of actions.  If we instead bucket the predictions and map predictions in the same bucket to the same action, the swap regret can be significantly reduced because the decision maker now sees many examples of the same bucket of predictions and will get a better sense of the true probability for the predictions in that bucket.

We view swap regret minimization as the \emph{goal} of making calibrated predictions, so it is best to achieve the original goal rather than a weaker relaxation; we view the response function as the \emph{method} used by decision makers to achieve the goal, so it can be flexible and need not be restricted to the exact best-response function.

Once we stop insisting on using the exact best-response function $r_Z^*$, the impossibility in \Cref{claim:impossibility} no longer applies and we can hope for a testable and actionable calibration measure for full swap regret.  We can formalize this objective in the following basic question.  

\begin{question}
\label{question}
Find a calibration measure $\CAL_T$ that satisfies the following properties:\footnote{To make the question easier, we allow the calibration measure to depend on the sample size $T$. This dependency is often not needed. The calibration measure $\scdl$ we introduce will not depend on $T$.}
    \begin{enumerate}
        \item \textbf{Actionable for \emph{some} response function.} For every decision task $Z = (A,U)$ with bounded utility function $U:A\times \{0,1\}\to [0,1]$ and every $\alpha \ge 0$, there exists a (possibly randomized)\footnote{We formally define the swap regret of a randomized response function in \Cref{def:swap}. The definition is in-expectation, but our calibration measure $\scdl$ will satisfy a stronger high-probability guarantee as well.} response function $r_{Z,\alpha}$ such that for every distribution $\cD$ on $[0,1]\times \{0,1\}$ satisfying $\CAL_T(\cD) \le \alpha$, the decisions made by applying $r_{Z,\alpha}$ to the prediction have low swap regret.  Formally, it holds that $\sr_Z(r_{Z,\alpha},\cD)\le \alpha$.  Importantly, the response function cannot depend on the distribution $\cD$, which the decision maker does not know. 
        
        \item \textbf{$\varepsilon(T)$-testable for calibrated distributions.} Let $\cD$ be a \emph{calibrated} distribution on $[0,1]\times \{0,1\}$. Let $S_T$ be a sample of $T$ i.i.d.\ points from $\cD$. It holds that $\E[\CAL_T(S_T)] \le \varepsilon(T)$.  We want $\varepsilon(T)$ as small as possible.
    \end{enumerate}
\end{question}

Intuitively, the strictest definition of actionability requires the response function to be exact best-response, which interprets the prediction $p$ as an exact probability and then chooses a response by calculating the utility-maximizing action, and requires that decisions made in this way have low swap regret.  However, this strict definition leads to the impossibility result of \Cref{claim:impossibility}.  In our relaxation, we get to specify some alternative way of using the predicted probabilities to make decisions and provide bounds on the swap regret of those decisions.  As discussed above, our alternative response function will roughly be to discretize the probabilities to a suitable precision and then apply best response, which we believe is a natural way of interpreting predicted probabilities.

\paragraph{Prior work on answers to \Cref{question}.}
Many existing calibration measures satisfy the two properties in \Cref{question} with vanishing estimation error $\varepsilon(T) \to 0$ as $T \to +\infty$. 
The best error rate we can get using measures proposed in prior work is $\varepsilon(T) = O(T^{-1/3})$, achieved by a binned variant of $\ECE$. Specifically, 
we choose $m \approx T^{1/3}$ and define a discretization function $b_m:[0,1]\to \{0, 1/m,2/m,\ldots,1\}$ that rounds predictions $p\in [0,1]$ down to the nearest multiple of $1/m$.  Given a distribution $\cD$ of prediction-outcome pairs $(p,y)$, we let $\cD_m$ denote the distribution of the discretized predictions and outcomes $(b_m(p),y)$. We define \emph{binned ECE} ($\bECE_m$) by applying $\ECE$ to the discretized distribution $\cD_m$ and multiplying the result by $2$. That is, $\bECE_m(\cD):= 2\,\ECE(\cD_m)$.

This definition of $\bECE_m$ is actionable when decision makers use a response function $r_Z$ that is slightly different from the best-response function $r_Z^*$.
Specifically, when a decision maker sees a prediction $p$, they should apply the best-response function $r_Z^*$ \emph{to the discretized prediction $b_m(p)$}. That is, we use $r_Z:= r_Z^*\circ b_m$ as the response function for decision task $Z$. 
It is known that $2\,\ECE$ is an upper bound on $\cdl$, so
$\bECE_m(\cD) = 2\, \ECE(\cD_m)$ is a valid upper bound on the swap regret $\sr_Z(r_Z^*,\cD_m)$, which is equal to $\sr_Z(r_Z,\cD)$ because the discretization $b_m$ in $r_Z = r_Z^*\circ b_m$ turns $\cD$ into $\cD_m$. This shows that $\bECE$ satisfies the actionability requirement of \Cref{question}.

The testability $\varepsilon(T) = O(T^{-1/3})$ of $\bECE_m$ can be proved by showing that the expected value of $\bECE_m(S)$ for a sample $S$ of $T$ i.i.d.\ points from any calibrated distribution $\cD$ is $O(T^{-1/3})$. Roughly, this can be obtained by summing up the discretization error $O(1/m) =O (T^{-1/3})$ and a sampling error which is of order $O(\sqrt{m/T}) = O(T^{-1/3})$. Choosing $m \approx T^{1/3}$ balances the two terms and achieves the optimal testability for binned $\ECE$. The discretization $b_m$ can be replaced by adding random noise to the predictions, where the absolute value of the noise is around $1/m\approx T^{-1/3}$. This leads to a calibration measure called \emph{smooth ECE} introduced by \cite{smoothECE} that is also actionable and achieves the same $O(T^{-1/3})$-testability.

In contrast, the cutoff calibration error \cite{test-action} has stronger testability $O(T^{-1/2})$, but it is not fully actionable because it only bounds swap regret for monotone swaps.  
For general swap functions, Roth and Shi \cite{Roth-Shi} show an actionability guarantee of the cutoff calibration error that degrades with the action space size $|A|$. That is, the swap regret in a downstream decision task with action space $A$ and a bounded utility function $U:A\times \{0,1\}\to [0,1]$ is at most $O(|A|)$ times the cutoff calibration error.
Due to the dependency on $|A|$, this is weaker than the actionability guarantee in \Cref{question}, where we require a uniform upper bound for all utility-bounded decision tasks independent of the number of actions.
Our goal is to achieve the same testability as cutoff calibration error while satisfying the stronger actionability requirement of \Cref{question}, with no dependency on $A$ and no restriction on the swap family.

The smooth calibration error $\smCE$ \cite{smooth,utc} also achieves the desired $O(T^{-1/2})$-testability. 
However, it is actionable only \emph{after we take its square root} \cite{smooth-decision,importance}.
Taking the square root compromises its testability from $O(T^{-1/2})$ to $O(T^{-1/4})$, which becomes even worse than the $O(T^{-1/3})$-testability guarantee of binned $\ECE$. 
The square root is indeed necessary to make the smooth calibration error actionable, even when general response functions are allowed. This is shown by an analysis in \cite{smooth-decision}. We present a simplified version of this analysis in \Cref{sec:quadratic-smooth}.
Similarly, in \Cref{thm:gap-cutoff} we show that if one wishes to make the cutoff calibration error fully actionable by taking its $\alpha$-th power, it is necessary to take $\alpha \le 2/3$ even when arbitrary response functions are allowed, compromising its testability from $O(T^{-1/2})$ to $O(T^{-1/3})$. For a summary of calibration measures satisfying \Cref{question}, see \Cref{tab:intro}.

\begin{table}[t]

\caption{$\scdl$ achieves full actionability and near-optimal testability. {\color{blue} Blue} text denotes the most desirable guarantee, while black text denotes a suboptimal guarantee.}
\label{tab:intro}
\centering
\begin{tabular}{lll}
\toprule
Measure & Actionability & Testability \\
\midrule
$\cdl$  & {\color{blue} Full} & $\Theta(1)$ \\
$\smCE$  & None & ${\color{blue} \Theta(T^{-1/2})}$ \\
$\smCE^{1/2}$    & {\color{blue} Full} & $\Theta(T^{-1/4})$ \\
$\cutoff$  & Monotone & ${\color{blue} \Theta(T^{-1/2})}$ \\
$\cutoff^\alpha$ & {\color{blue} Full only for $\alpha \le 2/3$} & $\Theta(T^{-\alpha/2}) = \Omega(T^{-1/3})$\\
Binned/Smooth \textsf{ECE}   & {\color{blue} Full} & $\Theta(T^{-1/3})$ \\
\hline
$\scdl$ (Our Work)   & {\color{blue} Full} & ${\color{blue} O(T^{-1/2}\log T)}$  \\
\bottomrule
\end{tabular}
\vspace{2pt}

\end{table}

\subsection{Our results: a calibration measure with nearly optimal actionability and testability}
\label{sec:contribution}
Our first contribution is an improved solution to \Cref{question} achieving $O(T^{-1/2}\log T)$-testability. We also prove a lower bound of $\Omega(T^{-1/2})$ showing that our solution is tight up to a logarithmic factor (\Cref{thm:lb}).

We prove our upper bound by designing a new calibration measure which we term \emph{Soft-Binned Calibration Decision Loss ($\scdl$)}. In addition to solving \Cref{question}, we show that $\scdl$ satisfies the following desired properties:
\begin{enumerate}
    \item $\scdl$ is \textbf{testable} not just for calibrated distributions, but for every distribution $\cD$. For a sample $S$ of $T$ i.i.d.\ points drawn from $\cD$,
    \begin{equation}
    \label{eq:intro-testable-general}
    \E|\scdl(S) - \scdl(\cD)|\le O(T^{-1/2} \log^{3/2} T).
    \end{equation}
    Moreover, the testability of $\scdl$ holds not just in expectation, but also with high probability. We formally state and prove the testability of $\scdl$ in \Cref{thm:testable}.  When $\scdl(\cD) = 0$ we prove a slightly stronger bound that $\E[\scdl(S)] = O(T^{-1/2} \log T)$.

    \item $\scdl$ is \textbf{actionable} for the full swap regret with a simple response function $r_{Z,\alpha}$. Similarly to our earlier discussion about binned ECE, we define $r_{Z,\alpha}(p)$ by applying the best-response function $r_Z^*$ to a discretized version $p'$ of the given prediction $p$. The discretization is very basic and does not depend on the decision task $Z$. We discuss the actionability of $\scdl$ in more detail in \Cref{sec:technical}.

    \item $\scdl$ is \textbf{sound and complete}. That is, a distribution $\cD$ satisfies $\scdl(\cD) = 0$ if and only if it is perfectly calibrated.

    \item $\scdl$ is a \textbf{consistent} calibration measure in the framework of \cite{utc}. That is, $\scdl$ is polynomially, in fact quadratically, related to the \emph{lower distance to calibration} $\distcal$ (\Cref{thm:smce}):
    \begin{equation}
    \label{eq:intro-relation}
    \Omega(\distcal(\cD)^2) \le \scdl(\cD) \le O(\sqrt{\distcal(\cD)}).
    \end{equation}
    The lower distance to calibration $\distcal$ (\Cref{def:distcal}) is introduced by \cite{utc} as a sound and complete calibration measure. Thus the quadratic relationship above implies the soundness and completeness of $\scdl$. We also show that both the upper and lower bounds in \eqref{eq:intro-relation} are tight up to constant factors (\Cref{remark:tight}).
    
    \item $\scdl$ is \textbf{continuous} as a function of the predictions. That is, small changes to the predictions can only cause small changes in $\scdl$. In \Cref{thm:continuity}, we prove that 
    $\scdl$ is $\frac12$-H\"older continuous.
\end{enumerate}
We complement our theoretical results with experiments following the setup of \cite{test-action}. On the testability side, $\scdl$ exhibits the smallest estimation variance across most settings, suggesting it can be estimated more reliably from finite samples than the other measures. On the actionability side, we evaluate the correlation between each calibration measure and swap regret across a range of prediction distributions. For a predictor that corresponds to a monotone optimal swap function, $\scdl$, $\ECE$ and $\cutoff$ correlate well with swap regret. However, when the optimal swap function is not monotone, $\cutoff$ loses its correlation with swap regret while $\scdl$ continues to track it, illustrating the advantage of not restricting the swap function class.%

\subsection{Technical overview}
\label{sec:technical}
\paragraph{Definition of $\scdl$.} 
We define $\scdl$ as a soft-binned variant of $\cdl$. Similarly to how binned ECE is actionable and testable, the actionability of $\scdl$ is inherited from the actionability of $\cdl$, whereas the testability of SCDL comes from binning, which is a natural idea that has been used widely in practice.
Our main technical contribution lies in proving the concrete, near-optimal testability of $\scdl$ using a self-balancing binning granularity, which we discuss below.

We start by defining the quantity $\scdl_m$, for a positive integer $m$.  First we discretize predictions into a finite grid $G_m:=\{0,1/m,2/m,\ldots,1\}$, where we view each grid point as a bin. Specifically, given a prediction $p$, we randomly round it to the closest multiple of $1/m$ either above or below, which we call $p'$, and we choose the probability of rounding up or down so that $\E p' = p$. For every distribution $\cD$ of $(p,y)\in [0,1]\times \{0,1\}$, we use $\cD_m$ to denote the distribution of $(p',y)\in G_m\times \{0,1\}$ where we use rounded predictions.

Given this discretization, the natural choice is to define $\scdl_m(\cD) = \cdl(\cD_m)$ to be the calibration decision loss of the discretized predictions.  While this is almost what we do, we make two technical modifications.  The first modification is to replace $\cdl(\cD_m)$ with an explicit formula that gives a factor-$2$ approximation to $\cdl$ and is easier to analyze (\Cref{lm:cdl}).  The second modification is to remove small bin-wise prediction bias up to $1/m$.  We make the second modification because the discretization can cause even a perfectly calibrated distribution to have nonzero $\cdl(\cD_m)$, and we want to subtract this rounding error off so that $\scdl_m(\cD) = 0$ when the predictions are perfectly calibrated. This modification makes $\scdl_m(\cD)$ non-decreasing as $m = 2^k$ grows as a power of two, which will help us choose the right discretization granularity $m$ to establish testability.

For any choice of $m$, the quantity $2(\scdl_m(\cD) + 1/m)$ satisfies the actionability requirement of \Cref{question}. To see why this is the case, first we have $\scdl_m(\cD) \approx \cdl(\cD_m)$, and $\cdl(\cD_m)$ is a valid upper bound on the swap regret $\sr_Z(r_Z^*,\cD_m)$ of the best-response function $r_Z^*$ applied to the discretized distribution $\cD_m$.  Next, the response function we define will take a prediction $p$, discretize it in the same way to obtain $p'$, and then apply the best-response function to $p'$.  Therefore, the swap regret of this new reponse function on the original distribution $\cD$ is the same as the swap regret of the best-response function on $\cD_m$.  The factor $2$ is due to the approximation formula for $\cdl$ and the extra $1/m$ term is to correct for the rounding error that we removed from the definition of $\scdl_m$. 

To complete the definition of $\scdl$, it remains to choose the value of $m$ appropriately. While the choice of $m$ does not affect the actionability of $2(\scdl_m(\cD) + 1/m)$, it makes a crucial difference in terms of testability. Indeed, $\cdl$ can be viewed as a special case of $\scdl_m$ with $m = +\infty$, but $\cdl$ is not testable. It turns out that to obtain our $\varepsilon(T) = O(T^{-1/2}\log T)$-testability guarantee, it suffices to choose $m \approx 1/\varepsilon(T)$. This choice, however, has the disadvantage of depending on the sample size $T$. Also, as we will discuss shortly, it only provides testability for calibrated distributions $\cD$ and does not satisfy the testability in \eqref{eq:intro-testable-general} for general distributions. Instead, we choose $m$ in a self-balancing way so that $\scdl_m(\cD)$ matches the rounding error $1/m$. Concretely, we choose $m = 2^k$ as the power of $2$ that minimizes $\max\{\scdl_m(\cD),1/m\}$ and define $\scdl(\cD)$ as this minimum value. We show that $\scdl_{2^k}(\cD)$ is an increasing function of $k$ whereas $1/m$ is a decreasing function. Thus, by minimizing the maximum of the two quantities, our choice of $m$ ensures a balance between them, and the final error $\scdl(\cD)$ is guaranteed to lie in the interval $[1/m,2/m)$.  This idea of choosing $m$ to balance $1/m$ and the binned calibration error is inspired by the \emph{interval calibration error} from \cite{utc} and it turns out to be essential for establishing our testability guarantees. 

Formally, we define $\scdl$ as follows.
\begin{definition}[Soft-Binned CDL $\scdl(\cD)$]
\label{def:scdl}
    Let $\cD$ be a distribution of $(p,y)\in [0,1]\times \{0,1\}$.  For every positive integer $m$, we define $\cD_m$ as the distribution of $(p',y)$, where $p'$ is obtained by randomly rounding $p$ up or down to the closest multiple of $1/m$ while ensuring $\E[p'|p,y] = p$.
    For every $i = 0,\ldots,m$, define the weight $\pi_i$ and the true conditional probability $q_i$ as follows:
    \[
    \pi_i:= \Pr\nolimits_{(p',y)\sim \cD_m}[p' = i/m], \quad q_i:= \E_{(p',y)\sim \cD_m}[y|p' = i/m].
    \]
For every $i = 0,\ldots,m$, define the $m$-soft-binned CDL $\scdl_m(\cD)$ as the following quantity that approximates $\cdl(\cD_m)$ (see \Cref{lm:v-shape}), where $x_+$ denotes $\max\{x,0\}$:
\[
\scdl_m(\cD):= \max_{i = 0,\ldots,m}\left(\sum\nolimits_{j = 0,\ldots,i} \pi_j(q_{j} - (i+1)/m)_+ + \sum\nolimits_{j=i+1,\ldots,m}\pi_j(i/m - q_j)_+\right).
\]
Finally, we define the soft-binned CDL of $\cD$ by choosing $m$ as a power of two that balances $1/m$ and $\scdl_m(\cD)$. That is,
$
\scdl(\cD) := \inf\nolimits_{m \in \{ 2^1,2^2,\ldots\}}\max\{\scdl_m(\cD),1/m\}.
$
\end{definition}

\paragraph{Actionability of $\scdl$.} 
    As discussed, the actionability of $\scdl$ follows naturally from its definition.
    Suppose we are given a predictor where the joint distribution $\cD$ of $(p,y)$ satisfies $\scdl(\cD) \le \alpha$. Here is how we use it to achieve low swap regret on a decision task $Z = (A,U)$ with a bounded utility function $U:A\times \{0,1\}\to [0,1]$.  If $\alpha = 0$, then $\cD$ is guaranteed to be perfectly calibrated and we simply use the best-response function $r_{Z,0} := r_Z^*$ to achieve zero swap regret $\sr_Z(r_Z^*,\cD) = 0$. If $\alpha > 0$, we find $m\in \{2^1,2^2,\ldots,\}$ such that $\alpha\in [1/m,2/m)$ and randomly round the prediction $p$ up or down to the closest multiple $p'$ of $1/m$ while ensuring $\E[p'] = p$. We take the best action $a = r_Z^*(p')$ based on the rounded prediction $p'$. That is, $r_{Z,\alpha}(p) := r_Z^*(p')$. We prove the following swap regret guarantee in \Cref{thm:actionability}:
    \begin{equation}
    \label{eq:intro-actionable}
    \sr_Z(r_{Z,\alpha},\cD) \le 2\,\scdl(\cD) + 2/m \le 4\alpha.
    \end{equation}
    This implies that $\scdl$ satisfies the full actionability required in \Cref{question}. Note that we incur an additional factor of $4$, so we need to set $\CAL = 4\,\scdl$ in \Cref{question} to satisfy the actionability requirement as stated. We choose not to include this factor of $4$ in the definition of $\scdl$ for a clearer presentation.  

    The response function $r_{Z,\alpha}$ that achieves the swap regret guarantee \eqref{eq:intro-actionable} is randomized due to randomness in the rounding procedure. 
    We formally define the swap regret of a randomized response function $r:[0,1]\to \Delta(A)$ in \Cref{def:swap}, where the regret against each swap function $\sigma:A\to A$ is calculated in expectation over actions sampled from $r(p)$.
    In \Cref{thm:actionability-hp}, we show that the actionablity guarantee of $\scdl$ holds  not just in expectation, but also with high probability. 
    
    We remark that we can replace the randomized rounding in the definition of $\scdl$ with a deterministic rounding procedure, where every $p$ is rounded to the largest multiple $p'$ of $1/m$ that does not exceed $p$. This ``hard-binned'' variant preserves all the testability and actionability properties of $\scdl$ following the same analysis in our work. However, this variant does not enjoy the continuity of $\scdl$. It is an interesting question for future research to understand the power of other response functions that do not require a discretization of the prediction space, such as quantal response functions and other smoothed best-response functions.

    \paragraph{Testability of $\scdl$.} Proving the testability of $\scdl$ is the most technically challenging part of our work. To bound the estimation error $|\scdl(S) - \scdl(\cD)|$ between the full distribution $\cD$ and a random sample $S$, the first step is to bound the estimation error for the conditional expectation $\E[y|p' = i/m]$ for each bin $i/m$, where $p'$ is the rounded version of $p$. Roughly speaking, a standard concentration argument shows that on average, each bin incurs an estimation error in the order of $\sqrt{m/T}$. Using this analysis, to get our final testability guarantee around $T^{-1/2}$, we are only allowed to have roughly $m = O(1)$ bins, in which case the rounding error $1/m$ will dominate, resulting in a large testability error.

To improve this basic analysis, we use a key result from Hu and Wu \cite{cdl} who showed that for $\cdl$, the errors from all bins do not contribute equally to the final estimation error. When $\cD$ is calibrated, they show that, roughly, only a $\sqrt{m/T}$ fraction of the bin-wise estimation errors contribute to the final error. This analysis shows that the estimation error reduces to roughly $m/T$. When we choose $m \approx T^{1/2}$, this $m/T$ estimation error and the rounding error $1/m$ are both of order $T^{-1/2}$.

However, this argument is limited to calibrated distributions $\cD$, and substantial new ideas are needed to prove the testability result for a general distribution $\cD$. We show that as $\scdl(\cD)$ grows, the estimation error deteriorates smoothly by roughly an additive $O(\sqrt{m\cdot \scdl(\cD)/T})$ term. Importantly, since $\scdl(\cD) \le 1$, this bound is better than the naive bound of $O(\sqrt{m/T})$.  However, for a highly miscalibrated distribution $\cD$ with $\scdl(\cD) = \Omega(1)$, the two bounds are of the same order and we can no longer choose $m \approx T^{1/2}$ to get the desired testability bound.  Instead, we would like to choose $m$ as a function of $\scdl(\cD)$, so that when $\scdl(\cD)$ is large, we can compensate for higher estimation error by taking fewer bins.  
This is accomplished by our self-balancing choice of $m$, which ensures that $m\cdot \scdl(\cD) = O(1)$ for \emph{all} distributions $\cD$.
Plugging this into the error bound gives the desired $\approx \sqrt{1/T}$ error. We summarize this result in the following theorem:
\begin{theorem}[$\scdl$ is testable, see \Cref{sec:testability} for proof of a stronger high-probability result]
\label{thm:testable-main}
    Let $S$ be a sample of $T\ge 2$ data points $(p_1,y_1),\ldots,(p_T,y_T)$ drawn i.i.d.\ from a distribution $\cD$ on $[0,1]\times \{0,1\}$. Then, in expectation over $S$, we have the bound
$
    \E|\scdl(S) - \scdl(\cD)| \le O(T^{-1/2}\log ^{3/2} T).
$
Moreover, when $\scdl(\cD) = 0$, we have $\E[\scdl(S)] = O(T^{-1/2}\log T)$.
\end{theorem}

\subsection{Related work}
\label{sec:related-work}

Recent research on calibration measures is largely motivated by 
Błasiok, Gopalan, Hu, and Nakkiran \cite{utc}, who studied the design of calibration measures in a systematic way. They introduce the distance to calibration as a central notion and identifies other calibration measures as \emph{consistent} if they are polynomially related to the distance to calibration. Our calibration measure $\scdl$ is quadratically related to the distance to calibration \eqref{eq:intro-relation}, so it is consistent according to the definition in \cite{utc}. Consistent calibration measures have been used to explain practical phenomena in deep learning (see e.g.\  \cite{when}).

Closely related to our goal of swap regret minimization for downstream decision makers, \emph{omniprediction} is the task of 
making predictions that allow a large family of downstream decision makers to achieve better or similar utility compared to a family of benchmark decision functions \cite{omni}. Although omniprediction does not require a swap regret guarantee, notions of calibration have been extensively used for constructing omnipredictors \cite{loss-oi,constrained-omni,performative,high-dim,equiv,omni-regression,oracle-omni,sim-mim,optimal-omni,lrs,hu2026simultaneous}.
Besides omniprediction, calibration and its variants have played important roles in the study of algorithmic fairness \cite{mc} and distributional robustness \cite{universal}.

Hu and Wu \cite{cdl} introduced the decision-theoretic calibration measure CDL motivated by the classic question of making calibrated sequential predictions \cite{asymptotic-calibration,QV,breaking-barrier,peng2025high,fishelson2026highdimensional} and the task of simultaneous online regret minimization for all downstream decision tasks \cite{u-cal,Roth-Shi}. Since \cite{cdl}, many recent works aim to find decision-theoretically meaningful calibration measures with additional properties. Besides the work on testability which we have discussed, Qiao and Zhao \cite{truthful-QZ} introduce a decision-theoretic calibration measure that is also approximately \emph{truthful}, where a truthful calibration measure incentivizes predictors to report true probabilities \cite{truthful-HQYZ,truthful}. 
The step calibration error used in \cite{truthful-QZ} can be viewed as equivalent to the cutoff calibration error from \cite{test-action}, although the two works are concurrent and independent.
The same calibration measure was also independently found to be useful for achieving omniprediction \cite{optimal-omni}.

\subsection{Paper organization}
The rest of the paper is organized as follows. We introduce basic notation as well as definitions and properties of previous calibration measures in \Cref{sec:preli}. In \Cref{sec:basic} we prove basic properties of $\scdl$ that will be used in later sections, where we prove the actionability (\Cref{sec:actionability}), continuity (\Cref{sec:continuity}), consistency (\Cref{sec:consistency}) and testability (\Cref{sec:testability}) of $\scdl$. In \Cref{sec:lb} we prove the $\Omega(T^{-1/2})$ lower bound for \Cref{question}. Our experimental results are in \Cref{sec:experiments}. We show that previous calibration measures do not give optimal solutions to \Cref{question} in \Cref{sec:lb-smooth,sec:lb-cutoff}.

\section{Preliminaries}
\label{sec:preli}
For every $x\in \R$, we use $x_+$ to denote $\max\{x,0\}$. We use $\ind[E]$ to denote the 0-1 indicator for a statement $E$, where $\ind[E] = 1$ if $E$ is true, and $\ind[E] = 0$ if $E$ is false. For every $p\in [0,1]$, we use $\Ber(p)$ to denote the Bernoulli distribution with mean $p$. We use $\Delta(A)$ to denote the set of all probability distributions on an action set $A$.

Below we present the formal definitions of existing calibration measures that are most related to our work. We also include key properties of these measures that will be useful for proving properties of $\scdl$.
\begin{definition}[Expected Calibration Error \cite{foster1997calibrated}]
    Let $\cD$ be a distribution of $(p,y)\in [0,1]\times \{0,1\}$. The Expected Calibration Error of $\cD$ is defined as $\ECE(\cD) = \E_\cD|q - p|$, where $q$ is the conditional expectation $\E_\cD[y\mid p]$, which is a function of $p$.
\end{definition}
\begin{definition}[Calibration Decision Loss \cite{cdl}]
\label{def:cdl}
Let $\cD$ be a distribution of $(p,y)\in [0,1]\times \{0,1\}$. The Calibration Decision Loss of $\cD$ is defined as
    \[
    \cdl(\cD):= \sup\nolimits_Z \sr_Z(r_Z^*,\cD),
    \]
    where the supremum is over all decision tasks $Z = (A,U)$ with an arbitrary action space $A$ and a bounded utility function $U:A\times \{0,1\}\to [0,1]$.
\end{definition}
The following relationship between $\cdl$ and $\ECE$ is established in \cite{cdl}:
\begin{lemma}[\cite{cdl}]
\label{lm:cdl-ece}
    For every distribution $\cD$ on $[0,1]\times \{0,1\}$, it holds that
\[
\ECE(\cD)^2 \le \cdl(\cD) \le 2\, \ECE(\cD).
\]

\end{lemma}
The following lemma gives an explicit formula for computing $\cdl$ up to a factor of $2$:
\begin{lemma}[\cite{cdl}]
\label{lm:v-shape}
Let $\cD$ be a distribution of $(p,y)\in [0,1]\times \{0,1\}$. Let $q = \E[y|p]\in [0,1]$ be a function of $p$ representing the conditional expectation of $y$ given $p$. We have
    \[
    \cdl(\cD) \le 2 \sup\nolimits_{\mu\in [0,1]} \E[(q - \mu)_+ \ind[p \le \mu] + (\mu - q)_+ \ind[p > \mu]].
    \]
We also have the reverse inequality up to a factor of $2$:
\[
\cdl(\cD) \ge \sup\nolimits_{\mu\in [0,1]} \E[(q - \mu)_+ \ind[p \le \mu] + (\mu - q)_+ \ind[p > \mu]].
\]
\end{lemma}

\begin{definition}[Smooth Calibration Error \cite{smooth,utc}]
    Let $\cD$ be a distribution of $(p,y)\in [0,1]\times \{0,1\}$. The smooth calibration error $\smCE(\cD)$ is defined as
    \[
    \smCE(\cD):= \sup_{\textup{$1$-Lipschitz $w:[0,1]\to [-1,1]$}}\E_{(p,y)\sim\cD}[(y - p)w(p)],
    \]
    where the supremum is over all $1$-Lipschitz functions $w:[0,1]\to [-1,1]$.
\end{definition}

\begin{definition}[Lower Distance to Calibration \cite{utc}]
\label{def:distcal}
    Let $\cD$ be a distribution of $(p,y)\in [0,1]\times \{0,1\}$. The lower distance to calibration $\distcal(\cD)$ is defined as
    \[
    \distcal(\cD):= \inf_{\cJ}\E_{(p,p',y)\sim \cJ}[|p - p'|],
    \]
    where the infimum is over all joint distributions $\cJ$ of $(p,p',y)\in [0,1]\times [0,1]\times \{0,1\}$ such that the distribution of $(p,y)$ is $\cD$ and the distribution of $(p',y)$ is calibrated.
\end{definition}
The following theorems shows that $\smCE$ and $\distcal$ are constant-factor approximations of each other: 
\begin{theorem}[\cite{utc,importance}]
    For every distribution $\cD$ on $[0,1]\times \{0,1\}$,
    \[
    \frac 12\, \distcal(\cD) \le \smCE(\cD) \le 2\, \distcal(\cD).
    \]
\end{theorem}

\begin{definition}[Cutoff calibration error \cite{test-action}]
    Let $\cD$ be a distribution of $(p,y)\in [0,1]\times \{0,1\}$. The cutoff calibration error $\cutoff(\cD)$ is defined as
    \[
    \cutoff(\cD):= \sup_{0 \le a \le b \le 1}|\E_{(p,y)\sim \cD}[(y - p)\ind[a\le p \le b]]|.
    \]
\end{definition}

\section{Basic Properties of SCDL}
\label{sec:basic}
In this section we prove some basic properties of $\scdl$. These properties will be useful for us to show its actionability, testability, consistency, and continuity in later sections.

The following lemma gives an equivalent definition of the bin-wise weight $\pi_i$ and true conditional probability $q_i$ used in our definition of $\scdl$ (\Cref{def:scdl}).

\begin{lemma}
\label{lm:equiv-def}
    Let $m$ be a positive integer. For every $i = 0,1,\ldots,m$ and $p\in [0,1]$, define $w_i(p):= (1 - |mp - i|)_+\in [0,1]$. For every distribution $\cD$ of $(p,y)\in [0,1]\times \{0,1\}$ and every $i = 0,\ldots,m$, the $\pi_i$ and $q_i$ from \Cref{def:scdl} satisfies
    \[
    \pi_i:= \E_{(p,y)\sim \cD}[w_i(p)], \quad q_i:= \frac 1{\pi_i}\E_{(p,y)\sim \cD}[w_i(p)y].
    \]
\end{lemma}
\begin{proof}
    Let $p'$ be the rounded prediction in \Cref{def:scdl}. We have $\Pr[p' = i/m|p] = w_i(p)$. Plugging it into the definition of $\pi_i$ and $q_i$ completes the proof.
\end{proof}

It will be convenient to define the following intermediate notion for a single bin $i$:
\begin{definition}[$\scfdl_{m,i}(\cD)$]
\label{def:scdl-appendix}
    Let $\cD$ be a distribution of $(p,y)\in [0,1]\times \{0,1\}$. Let $m$ be a positive integer. For every $i = 0,1,\ldots,m$, define $\pi_i$ and $q_i$ as in \Cref{def:scdl}.
For every $i = 0,\ldots,m$, define the \emph{soft-binned calibration fixed decision loss} as follows:
\[
\scfdl_{m,i}(\cD):= \sum_{j = 0}^i \pi_j(q_{j} - (i+1)/m)_+ + \sum_{j=i+1}^m\pi_j(i/m - q_j)_+.
\]
\end{definition}
Combining \Cref{def:scdl-appendix} and \Cref{def:scdl}, we have
\[
\scdl_m(\cD) = \max_{i = 0,\ldots,m}\scfdl_{m,i}(\cD).
\]

The following lemma shows that $\scdl_m(\cD)$ is a non-decreasing function of $m$ when $m$ is a power of $2$.
\begin{lemma}
\label{lm:monotone}
    For every positive integer $m$ and every distribution $\cD$ on $[0,1]\times \{0,1\}$, $\scdl_{2m}(\cD) \ge \scdl_m(\cD)$.
\end{lemma}
\begin{proof}    
For the discretization in $2m$ bins, for every $i = 0, 1, \ldots, 2m$, define $\tilde w_i(p):= (1-|2mp-i|)_+$, $\tilde\pi_i := \E_{(p,y)\sim \cD}[\tilde w_i(p)]$ and $\tilde q_i := \frac{1}{\tilde \pi_i}\E_{(p,y)\sim \cD}[\tilde w_i(p) y]$. Then, we see that for any $p \in [0,1]$
\[
w_i(p) = \begin{cases}
\tilde w_0(p) + \frac{1}{2}\tilde w_1(p) & i = 0,\\
\tilde w_{2i}(p) + \frac{1}{2}\tilde w_{2i-1}(p) + \frac{1}{2}\tilde w_{2i+1}(p) & 1 \le i \le m-1,\\
\tilde w_{2m}(p) + \frac{1}{2}\tilde w_{2m-1}(p) & i = m.
\end{cases}
\]
Therefore, by taking the expectation of $w_i(p)$ we have that 
\[
\pi_i = \begin{cases}
\tilde \pi_0 + \frac{1}{2}\tilde \pi_1 & i = 0,\\
\tilde \pi_{2i} + \frac{1}{2}\tilde \pi_{2i-1} + \frac{1}{2}\tilde \pi_{2i+1} & 1 \le i \le m-1,\\
\tilde \pi_{2m} + \frac{1}{2}\tilde \pi_{2m-1} & i = m.
\end{cases}
\] 
Similarly, taking the expectation of $w_i(p)y$, we obtain
\[
q_i = \begin{cases}
\frac{1}{\pi_i}(\tilde \pi_0 \tilde q_0 + \frac{1}{2}\tilde \pi_1 \tilde q_1)& i = 0,\\
\frac{1}{\pi_i}(\tilde \pi_{2i}\tilde q_{2i} + \frac{1}{2}\tilde \pi_{2i-1}\tilde q_{2i-1} + \frac{1}{2}\tilde \pi_{2i+1}\tilde q_{2i+1} )& 1 \le i \le m-1,\\
 \frac{1}{\pi_i}(\tilde\pi_{2m} \tilde q_{2m}+ \frac{1}{2}\tilde \pi_{2m-1}\tilde q_{2m-1}) & i = m.
\end{cases}
\] 

For a fixed $i \in \{0,\ldots, m\}$ functions $(x-(i+1)/m)_+$ and $(i/m-x)_+$ are convex in $x$. Therefore, by Jensen's inequality for a fixed $i \in \{0,\ldots,m\}$ we have that
\begin{align*}
    &\scfdl_{m,i}(\cD) \\
    = {} &\sum_{j=0}^i \pi_j\left (q_j-\frac{i+1}{m}\right)_+ + \sum_{j=i+1}^m \pi_j\left(\frac{i}{m}-q_j\right)_+\\
    \leq {} & \tilde \pi_0 \left(\tilde q_0-\frac{i+1}{m}\right)_+ + \frac{1}{2}\tilde \pi_1\left(\tilde q_1-\frac{i+1}{m}\right)_+ \\
    & + \sum_{j=1}^i \Bigg( \frac{1}{2}\tilde \pi_{2j-1}\left(\tilde q_{2j-1}-\frac{i+1}{m}\right)_+ + \tilde \pi_{2j}\left(\tilde q_{2j} -\frac{i+1}{m}\right)_++\frac{1}{2}\tilde \pi_{2j+1}\left(\tilde q_{2j+1}-\frac{i+1}{m}\right)_+\Bigg)\\
    & + \sum_{j=i+1}^{m-1}\Bigg(\frac{1}{2}\tilde \pi_{2j-1}\left(\frac{i}{m}-\tilde q_{2j-1}\right)_++\tilde \pi_{2j} \left(\frac{i}{m}-\tilde q_{2j}\right)_+ + \frac{1}{2}\tilde \pi_{2j+1}\left(\frac{i}{m}-\tilde q_{2j+1}\right)_+\Bigg) \\
    & + \frac{1}{2}\tilde \pi_{2m-1}\left(\frac{i}{m}-\tilde q_{2m-1}\right)_+ + \tilde \pi_{2m}\left(\frac{i}{m} - \tilde q_{2m}\right)_+\\
     = {} & \sum_{j=0}^{2i} \tilde \pi_j\left(\tilde q_j-\frac{i+1}{m}\right)_+ + \frac{1}{2} \tilde \pi_{2i+1}\left(\tilde q_{2i+1} - \frac{i+1}{m}\right)_+ \\
     & + \frac{1}{2} \tilde \pi_{2i+1}\left(\frac{i}{m}-\tilde q_{2i+1}\right)_+ + \sum_{j=2i+2}^{2m} \tilde \pi_j\left(\frac{i}{m}-\tilde q_j\right)_+\\
     = {} & \sum_{j=0}^{2i} \tilde \pi_j\left(\tilde q_j-\frac{2i+2}{2m}\right)_+ + \frac{1}{2} \tilde \pi_{2i+1}\left(\tilde q_{2i+1} - \frac{2i+2}{2m}\right)_+ \\
     & + \frac{1}{2} \tilde \pi_{2i+1}\left(\frac{2i}{2m}-\tilde q_{2i+1}\right)_+ + \sum_{j=2i+2}^{2m} \tilde \pi_j\left(\frac{2i}{2m}-\tilde q_j\right)_+ \\
     \leq {} & \frac{1}{2}\scfdl_{2m,2i}(\cD)+ \frac{1}{2}\scfdl_{2m,2i+1}(\cD)\\
     \leq {} &\max \{\scfdl_{2m,2i}(\cD), \scfdl_{2m,2i+1}(\cD)\}.
\end{align*}

Hence, we conclude that 
\begin{align*}
    \scdl_m(\cD) =\max_{i=0,\ldots,m}\scfdl_{m,i}(\cD) \leq \max_{i=0,\ldots,2m}\scfdl_{2m,i}(\cD) = \scdl_{2m}(\cD).
\end{align*}
\end{proof}
In our definition of $\scdl$, we choose $m$ by minimizing $\max\{\scdl_m(\cD),1/m\}$. We give a characterization of the minimizing $m$ in \Cref{def:m*} and \Cref{lm:m*}.

\begin{definition}[Optimal Number of Bins]
\label{def:m*}
    Let $\cD$ be a distribution of $(p,y)\in [0,1]\times \{0,1\}$. We define $m^*(\cD)$ as the smallest number $m$ among $2^1, 2^2, \ldots,$ such that $\scdl_{2m}(\cD) \ge 1/m$. When such $m$ does not exist, we define $m^*(\cD) = +\infty$.
\end{definition}
\begin{lemma}
\label{lm:m*}
    Let $\cD$ be a distribution of $(p,y)\in [0,1]\times \{0,1\}$. If $m^*:= m^*(\cD) = +\infty$, then $\scdl(\cD) = 0$. If $m^* < +\infty$, we have
    \begin{enumerate}
        \item $\scdl(\cD) = \max\{\scdl_{m^*}(\cD), 1/m^*\}\in [1/m^*, 2/m^*)$;
        \item $\scdl_{2m^*}(\cD) \ge \scdl(\cD)$.
    \end{enumerate}
\end{lemma}
\begin{proof}
    If $m^* = +\infty$, by the monotonicity in \Cref{lm:monotone}, we have that for all $m$ in $2^1, 2^2, \ldots$ $\scdl_{m}(\cD) \leq \scdl_{2m}(\cD) <1/m$. Hence, $\scdl(\cD) = \inf_{m= 2^1, 2^2, \ldots}\max\{\scdl_m(\cD),1/m\} = \inf_{m=2^1, 2^2, \ldots}1/m = 0$.
    
    Now, we analyze the case for $m^* < +\infty$. First, we show that $\max\{\scdl_{m^*}(\cD), 1/m^*\} \in [1/m^*, 2/m^*)$. If $m^* \geq 4$, since $m^*$ is the smallest number among $2^1, 2^2, \ldots$ such that $\scdl_{2m^*}(\cD)\geq 1/m^*$, we have that for $m^*/2$
    \[
    \scdl_{m^*}(\cD) < \frac{2}{m^*}.
    \]
    Therefore, if $\scdl_{m^*}(\cD) \leq 1/m^*$, then $\max\{\scdl_{m^*}(\cD),1/m^*\} = 1/m^*$. Otherwise, $\max\{\allowbreak \scdl_{m^*}(\cD),1/m^*\} = \scdl_{m^*}(\cD) < \frac{2}{m^*}$. For $m^* = 2$, we notice that $\scdl_2(\cD) \le 1/2$. Thus, $\max\{\scdl_2(\cD),1/2\} = 1/2 \in [1/m^*,2/m^*)$. 

    Next, we show that $\scdl(\cD) = \max\{\scdl_{m^*}(\cD), 1/m^*\}$. For $m$ that is dyadic and $m > m^*$, by \Cref{lm:monotone}, we have that $\scdl_{m}(\cD) \geq \scdl_{2m^*}(\cD) \geq 1/m^*$ and $\scdl_{m}(\cD) \geq \scdl_{m^*}(\cD)$. Therefore, $\max\{\scdl_m(\cD),1/m\} \geq \max\{\scdl_{m^*}(\cD), 1/m^*\}$. For $m$ that is dyadic and $m < m^*$, it holds that $m \leq m^*/2$. Thus, $\max\{\scdl_m(\cD), 1/m\} \geq 1/m \geq 2/m^* > \max\{\scdl_{m^*}(\cD) , 1/m^*\}$.
    
    By \Cref{lm:monotone}, $\scdl_{2m^*}(\cD)\geq \scdl_{m^*}(\cD)$ and by the definition of $m^*$, we have that $\scdl_{2m^*}(\cD) \geq 1/m^*$. Thus, $\scdl_{2m^*}(\cD) \geq \scdl(\cD)$.
\end{proof}

The following lemma shows a basic relationship between $\scdl$ and $\cdl$.

\begin{lemma}
\label{lm:cdl}
Let $\cD$ be a distribution on $[0,1]\times \{0,1\}$.
    For every positive integer $m$, it holds that $\scdl_m(\cD) \le \cdl(\cD)$. Consequently, $\scdl(\cD) \le \cdl(\cD)$.
\end{lemma}
\begin{proof}
Let $\pi_i$ and $q_i$ be defined as in \Cref{def:scdl} for $i = 0,\ldots,m$.
To show $\scdl_m(\cD) \le \cdl(\cD)$, it suffices to show that for every $i = 0,\ldots,m$,
\begin{equation}
\label{eq:scdl-cdl-goal}
\sum_{j = 0}^i \pi_j(q_{j} - (i+1)/m)_+ + \sum_{j=i+1}^m\pi_j(i/m - q_j)_+ \le \cdl(\cD).
\end{equation}
Consider the joint distribution of $(p,y,q,p')\in [0,1]\times \{0,1\}\times [0,1]\times \{0,1/m,\ldots,1\}$, where $(p,y)\sim \cD$, $p'$ is the rounded version of $p$ as in \Cref{def:scdl}, and $q:= \E[y|p]$. We have 
\[
q_j = \E[y|p' = j/m] = \E\Big[\E[y|p' = j/m, p]\Big|p' = j/m\Big] = \E[q|p' = j/m],
\]
where the first equation is the law of total probability, and the second equation holds because $y$ is independent of $p'$ conditioned on $p$.

Since $x_+$ is a convex function of $x$, by Jensen's inequality,
\begin{align*}
(q_j - (i + 1)/m)_+ & \le \E[(q - (i+1)/m)_+|p' = j/m],\\
(i/m - q_j)_+ & \le \E[(i/m - q)_+|p' = j/m].
\end{align*}
Therefore,
\begin{align}
& \sum_{j = 0}^i \pi_j(q_{j} - (i+1)/m)_+ + \sum_{j=i+1}^m\pi_j(i/m - q_j)_+ \notag\\
\le {} & \E[(q - (i + 1)/m)_+\ind[p' \le i/m]] + \E[(i/m - q)_+\ind[p' \ge (i + 1)/m]].\label{eq:scdl-cdl-0}
\end{align}
For every $p\in [0,1]$, define 
\[
\beta(p):= \Pr[p' \le i/m|p] = \begin{cases}
    1, & \text{if $p \le (i-1)/m$},\\
    i + 1 - pm, & \text{if $p\in (i/m, (i+1)/m)$},\\
    0, & \text{if $p \ge (i + 1)/m$.}
\end{cases}
\]
We have $\Pr[p' \ge (i + 1)/m | p] = 1 - \beta(p)$. By \eqref{eq:scdl-cdl-0},
\begin{align}
& \sum_{j = 0}^i \pi_j(q_{j} - (i+1)/m)_+ + \sum_{j=i+1}^m\pi_j(i/m - q_j)_+ \notag \\
\le {} & \E[(q - (i + 1)/m)_+\beta(p)] + \E[(i/m - q)_+(1 - \beta(p))].\label{eq:scdl-cdl-1}
\end{align}
Let $\mu$ be drawn from the uniform distribution on $[i/m,(i + 1)/m]$, independent of $(p,q,p')$. For every $p\in [0,1]$, we have
\begin{equation}
\label{eq:unif-mu}
\Pr_\mu[p \le \mu] = \beta(p), \quad \Pr_\mu[p > \mu] = 1 - \beta(p).
\end{equation}
By \Cref{lm:v-shape},
\begin{align}
    & \cdl(\cD) \notag \\
    \ge {} & \E_\mu \E[(q - \mu)_+\ind[p \le \mu] + (\mu - q)_+\ind[p > \mu]]\notag \\
    \ge {} & \E_\mu \E[(q - (i + 1)/m)_+\ind[p \le \mu] + (i/m - q)_+\ind[p > \mu]]\notag \\
    = {} & \E[(q - (i + 1)/m)_+\beta(p) + (i/m - q)_+(1 - \beta(p))],\label{eq:scdl-cdl-2}
\end{align}
where we used \eqref{eq:unif-mu} to obtain the last equality. Combining \eqref{eq:scdl-cdl-1} and \eqref{eq:scdl-cdl-2} proves \eqref{eq:scdl-cdl-goal}.

    Now, we show that $\scdl(\cD) \leq \cdl(\cD)$. By \Cref{lm:m*}, we have that if $m^* = +\infty$, then $\scdl(\cD) = 0 \leq \cdl(\cD)$. Otherwise, we have that
    \[
    \scdl(\cD) \leq \scdl_{2m^*}(\cD) \leq \cdl(\cD).\qedhere
    \]
\end{proof}

\section{Actionability of SCDL} \label{sec:actionability}
We formally state and the prove the actionability of $\scdl$ in this section.

For every positive integer $m$, let $G_m$ denote the grid $\{0,1/m,\ldots,1\}$. 
When defining $\scdl$ in \Cref{def:scdl}, we rounded each prediction $p\in [0,1]$ to a grid point $p'\in G_m$. We formally define this rounding function as follows:

\begin{definition}[Randomized rounding function $\tau_m$]
\label{def:tau}
For every prediction $p\in [0,1]$, we let $\tau_m(p)$ be the distribution of the rounded prediction $p'\in G_m$ obtained by randomly rounding $p$ up or down to the closest multiple of $1/m$ while ensuring $\E[p'|p] = p$.
\end{definition}

\begin{definition}[Rounded distribution $\cD_m$]
\label{def:dm}
Given $(p,y)$ drawn from a distribution $\cD$ on $[0,1]\times \{0,1\}$, we draw $p'\sim \tau_m(p)$ and use $\cD_m$ to denote the distribution of $(p',y)$.
\end{definition}
The following lemma follows directly by combining \Cref{def:scdl} and \Cref{def:dm}.
\begin{lemma}
\label{lm:dm}
Let $\cD$ be a distribution on $[0,1]\times \{0,1\}$ and let $\cD_m$ be defined as in \Cref{def:dm}. We have
\[
\scdl_m(\cD_m) = \scdl_m(\cD).
\]
\end{lemma}

\begin{definition}[Randomized response function $r\circ \tau_m$]
    For every response function $r:[0,1]\to A$, we use $r\circ \tau_m:[0,1]\to \Delta(A)$ to denote the randomized response function where  $r\circ\tau_m(p)$ is the distribution of $r(p')$ for $p'\sim \tau_m(p)$.
\end{definition}

\begin{definition}[Swap regret of randomized response]
\label{def:swap}
Let $Z = (A,U)$ be a decision task with action space $A$ and utility function $U$.
    Given a distribution $\cD$ on $[0,1]\times \{0,1\}$ and a randomized response function $r:[0,1]\to \Delta(A)$ that maps a prediction $p\in [0,1]$ to a distribution over actions $a\in A$, we define the swap regret of $r$ on distribution $\cD$ as
    \[
    \sr_Z(r,\cD) := \sup_{\sigma:A\to A}\E_{(p,y)\sim \cD, a\sim r(p)} [u(\sigma(a),y) - u(a,y)].
    \]
\end{definition}
We are now ready to formally state and prove the actionability of $\scdl$:
\begin{theorem}[$\scdl$ is actionable]
\label{thm:actionability}
Let $\cD$ be a distribution of $(p,y)\in [0,1]\times \{0,1\}$. Assume $\scdl(\cD) < 2/m$ for some $m\in M:= \{2^1,2^2,\ldots\}$. Then we have
\[
\sup_Z\sr_Z(r^*_Z\circ \tau_m,\cD) = \cdl(\cD_m) \le 2\,\scdl_m(\cD) + 2/m \le 2\,\scdl(\cD) + 2/m,
\]
where the supremum is over all decision tasks $Z = (A,U)$ with a bounded utility function $U:A\times \{0,1\}\to [0,1]$.
\end{theorem}
We prove the theorem using the following lemma. 

\begin{lemma}
\label{lm:cdlm}
Let $\cD$ be a distribution on $G_m\times \{0,1\}$. We have $\cdl(\cD) \le 2\, \scdl_m(\cD) + 2/m$.
\end{lemma}
\begin{proof}
        By \Cref{lm:v-shape}, 
    \begin{align*}
        \cdl(\cD) &\le 2\, \sup\nolimits_{\mu\in [0,1]} \E[(q - \mu)_+ \ind[p \le \mu] + (\mu - q)_+ \ind[p > \mu]].
    \end{align*}

We partition the space $[0,1]$ into $m$ intervals $B_0,\ldots,B_{m-1}$ as follows. We define $B_i:= [i/m,(i + 1)/m)$ for $i = 0,\ldots,m -2$ and $B_{m - 1} = [(m-1)/m,1]$.
It remains to show that for every $i \in \{0,\ldots,m-1\}$ and every $\mu\in B_i$, it holds that
\begin{equation}
\label{eq:cdlm-1}
\E[(q - \mu)_+ \ind[p \le \mu] + (\mu - q)_+ \ind[p \ge \mu]] \le \sum_{j = 0}^i \pi_j(q_{j} - (i+1)/m)_+ + \sum_{j=i+1}^m\pi_j(i/m - q_j)_+ + \frac 1m.
\end{equation}
By our assumption, the predictions $p$ lie on the grid $G_m = \{0,1/m,\ldots,1\}$, so we have
\[
\E[(q - \mu)_+\ind[p\le \mu]] = \sum_{j = 1}^i\pi_j(q_j - \mu)_+ \le \sum_{j = 1}^i\pi_j\left((q_j - (i + 1)/m)_+ + \frac 1m\right).
\]
The inequality above holds also in the special case where $i = m - 1$  and $\mu = 1\in B_i$, because in this case the left-hand side is zero: $(q - \mu)_+ = (q - 1)_+ = 0$.
Similarly,
\[
\E[(\mu - q)_+ \ind[p > \mu]] = \sum_{j = {i + 1}}^m \pi_j(\mu - q_j)_+ \le  \sum_{j = {i + 1}}^m \pi_j\left((i/m - q_j)_+ + \frac 1m\right).
\]
Combining these two inequalities with the fact that $\sum _{j = 0}^m\pi_j = 1$ proves \eqref{eq:cdlm-1}.

\end{proof}
\begin{proof}[Proof of \Cref{thm:actionability}]
    By the definition of $\cdl$ (\Cref{def:cdl}), 
    \begin{equation}
    \label{eq:action-1}
    \sup_Z\sr_Z(r^*_Z\circ \tau_m,\cD) = \cdl(\cD_m).
    \end{equation}
By \Cref{lm:cdlm} and \Cref{lm:dm}, we have
\begin{equation}
\label{eq:action-2}
\cdl(\cD_m) \le 2\, \scdl_m(\cD) + \frac 2m.
\end{equation}
By \Cref{lm:m*}, we have $m^*:= m^*(\cD) \ge m$, and
\begin{equation}
\label{eq:action-3}
\scdl_m(\cD) \le \scdl_{m^*}(\cD) \le \scdl(\cD).
\end{equation}
Combining \eqref{eq:action-1}, \eqref{eq:action-2}, and \eqref{eq:action-3} completes the proof.
\end{proof}

\begin{theorem}[$\scdl$ is actionable with high probability]
\label{thm:actionability-hp}
Let $(p_1,y_1),\ldots,(p_T,y_T)\in [0,1]\times \{0,1\}$ be an arbitrary sequence of prediction-outcome pairs and let $\cD$ be the uniform distribution over these pairs. Let $m\in \{2^1,2^2,\ldots\}$ such that $\scdl(\cD)< 2/m$. For $\delta\in (0,1/(2m))$, assume
\[
\frac 1m = \Omega\left(\sqrt{\frac{\log(1/\delta)}{T}}\right).
\]
For every $t = 1,\ldots,T$, we independently round $p_i$ to $p_i'\in G_m$ using the randomized rounding function $\tau_m$ in \Cref{def:tau} and let $\cD'$ be the uniform distribution over $(p_1',y_1),\ldots,(p_T',y_T)$. Then with probability at least $1-\delta$,
\[
\sup_Z\sr_Z(r_Z^*,\cD') \le 2\, \scdl(\cD) + \frac 2m + O\left(\sqrt{\frac{\log(1/\delta)}{T}}\cdot \log m\right),
\]
where the supremum is over all decision task $Z = (A,U)$ with a bounded utility function $U:A\times \{0,1\}\to [0,1]$.
\end{theorem}
\begin{proof}
    By \Cref{def:cdl}, we have
    \begin{equation}
    \label{eq:hp-1}
    \sup_Z\sr_Z(r_Z^*,\cD')  = \cdl(\cD').
    \end{equation}
    By \Cref{lm:cdlm}, we have
\begin{equation}
\label{eq:hp-2}
\cdl(\cD') \le 2\,\scdl_m(\cD') + \frac 2m.
\end{equation}
By \Cref{lm:hp-helper}, with probability at least $1-\delta$,
\begin{equation}
\label{eq:hp-3}
\scdl_m(\cD') \le \scdl_m(\cD) + O\left(\sqrt\frac{\log(1/\delta)}{T} \cdot \log m\right).
\end{equation}
    By \Cref{lm:m*}, we have $m^*:= m^*(\cD) \ge m$, and
\begin{equation}
\label{eq:hp-4}
\scdl_m(\cD) \le \scdl_{m^*}(\cD) \le \scdl(\cD).
\end{equation}
Combining \eqref{eq:hp-1}, \eqref{eq:hp-2}, \eqref{eq:hp-3} and \eqref{eq:hp-4} completes the proof.
\end{proof}

\section{Continuity of SCDL}
\label{sec:continuity}
We prove the following theorem showing that $\scdl$ is a continuous function of the predictions.
\begin{theorem}[Continuity]
\label{thm:continuity}
    Let $\cJ$ be a joint distribution of $(p_1,p_2,y)\in [0,1]\times [0,1]\times \{0,1\}$. Let $\cD_1$ be the distribution of $(p_1,y)$ and let $\cD_2$ be the distribution of $(p_2,y)$. We have
    \begin{align*}
    |\scdl(\cD_1)^2 - \scdl(\cD_2)^2| & \le O(\E_\cJ|p_1 - p_2|), \quad\text{and}\\
    |\scdl(\cD_1) - \scdl(\cD_2)|  & \le O\left(\sqrt{\E_\cJ|p_1 - p_2|}\right).
    \end{align*}
\end{theorem}
We need the following helper lemma:
\begin{lemma}
\label{lm:continuity}
    Let $\cJ$ be a joint distribution of $(p_1,p_2,y)\in [0,1]\times [0,1]\times \{0,1\}$. Let $\cD_1$ be the distribution of $(p_1,y)$ and let $\cD_2$ be the distribution of $(p_2,y)$. We have
    \[
    |\scdl_m(\cD_1) - \scdl_m(\cD_2)| \le O( m\cdot  \E_\cJ|p_1 - p_2|) \quad \text{for every $m = 0,1,\ldots$}.
    \]
\end{lemma}
\begin{proof}
    It suffices to show that for every $i = 0,\ldots,m$,
        \[
    |\scfdl_{m,i}(\cD_1) - \scfdl_{m,i}(\cD_2)| \le O( m\cdot  \E_\cJ|p_1 - p_2|).
    \]
    We first show that for any two predictions $p,p'\in [0,1]$, it holds that
    \[
    \sum_{i = 0}^m |w_i (p) - w_i(p')| \le O(m|p - p'|).
    \]
    This is because for all but at most $4$ choices of $i$ we have $w_i(p) = w_i(p') = 0$, where the $4$ choices are the two neighboring bins to $p$ and the two neighboring bins to $p'$. For these $4$ choices of $i$, we have $|w_i(p) - w_i(p')| \le O(m|p - p'|)$.
    
    Now we have
    \begin{align*}
        \sum_{i = 0}^m |\pi_i - \pi_i'| & = \sum_{i = 0}^m |\E_\cJ (w_i(p) - w_i(p'))|\\
        & \le \E_\cJ\sum_{i = 0}^m|w_i(p) - w_i(p')|\\
        & \le O(m\cdot \E_\cJ|p - p'|).
    \end{align*}
    Similarly,
        \begin{align*}
        \sum_{i = 0}^m |\pi_iq_i - \pi_i'q_i'| & = \sum_{i = 0}^m |\E_\cJ (w_i(p)y - w_i(p')y)|\\
        & \le \E_\cJ\sum_{i = 0}^m|w_i(p) - w_i(p')|y\\
        & \le O(m\cdot \E_\cJ|p - p'|).
    \end{align*}
    By the definition of $\scfdl$ and the triangle inequality,
    \begin{align*}
        & |\scfdl_{m,i}(\cD_1) - \scfdl_{m,i}(\cD_2)|\\
        \le {} & \sum_{j\le i}|(\pi_j q_j - \pi_j(i + 1)/m)_+ - (\pi_j' q_j' - \pi_j'(i + 1)/m  )_+|\\
        & + \sum_{j > i}|(\pi_ji/m - \pi_j q_j)_+ + (\pi_j'i/m - \pi_j'q_j')_+|\\
        \le {} & \sum_{j = 0}^m |\pi_jq_j - \pi_j'q_j'| + \sum_{j = 0}^m |\pi_j - \pi_j'|\\
        \le {} & O(m\cdot \E_\cJ|p - p'|).\qedhere
    \end{align*}
\end{proof}
\begin{proof}[Proof of \Cref{thm:continuity}]
Assume without loss of generality that $\E_\cJ|p_1 - p_2| > 0$ and $\scdl(\cD_1) \le \scdl(\cD_2)$. Define $m^*:= m^*(\cD_1)$. By \Cref{lm:continuity}, for an absolute constant $C > 0$ and every $m = 0,1,\ldots$ satisfying $m \le m^*$,
\begin{equation}
\label{eq:continuity-1}
\scdl_m(\cD_2) \le \scdl(\cD_1) + C m \cdot \E_\cJ|p_1 - p_2|.
\end{equation}
Choose $m$ to be smallest number among $2^1, 2^2,\ldots$ satisfying
\begin{equation}
\label{eq:m-smallest}
\scdl(\cD_1) + Cm \cdot \E_\cJ|p_1 - p_2| \ge \frac 1m.
\end{equation}
We have $m \le m^*$ because $\scdl(\cD_1) \ge 1/m^*$ (\Cref{lm:m*}).
By \eqref{eq:continuity-1},
\[
\scdl(\cD_2) \le \max\{\scdl_m(\cD_2), 1/m\} \le \scdl(\cD_1) + Cm \cdot \E_\cJ|p_1 - p_2|.
\]
Therefore,
\begin{align*}
    \scdl(\cD_2)^2 - \scdl(\cD_1)^2 & = (\scdl(\cD_2) - \scdl(\cD_1))(\scdl(\cD_1) + \scdl(\cD_2))\\
    & \le Cm\cdot \E_\cJ|p_1 - p_2|\ (2\cdot \scdl(\cD_1) + Cm\cdot \E_\cJ|p_1 - p_2|)\\
    & \le Cm\cdot \E_\cJ|p_1 - p_2| \cdot O(1/m) \tag{because $m$ is the smallest power of $2$ satisfying \eqref{eq:m-smallest}}\\
    & \le O(\E_\cJ|p_1 - p_2|).
\end{align*}
Finally,
\begin{align*}
(\scdl(\cD_2) - \scdl(\cD_1))^2 & \le (\scdl(\cD_2) - \scdl(\cD_1))(\scdl(\cD_2) + \scdl(\cD_1))\\ & = (\scdl(\cD_2)^2 - \scdl(\cD_1)^2)\\
& \le O(\E_\cJ|p_1 - p_2|).\qedhere
\end{align*}
\end{proof}

\section{Quadratic Relationship between SCDL and the Distance to Calibration}
\label{sec:consistency}
We prove the following relationship between $\scdl$ and the lower distance to calibration $\distcal$ (\Cref{def:distcal}).
\begin{theorem}[Relationship to $\distcal$]For every distribution $\cD$ on $[0,1]\times \{0,1\}$,
\label{thm:smce}
    \[
    \Omega(\distcal(\cD)^2) \le \scdl(\cD) \le O(\sqrt{\distcal(\cD)}).
    \]
\end{theorem}

\begin{proof}
For every positive integer $m$, we have
\begin{align*}
    \distcal(\cD) & \le \distcal(\cD_m) + 1/m\\ & \le \ECE(\cD_m) + 1/m \\
    & \le O\left(\sqrt{\cdl(\cD_m)}\right) + 1/m \tag{by \Cref{lm:cdl-ece}}\\
    & \le O\left(\sqrt{\max\{\scdl_m(\cD_m),1/m\}}\right) \tag{by \Cref{lm:cdlm}}\\
    & = O\left(\sqrt{\max\{\scdl_m(\cD),1/m\}}\right). \tag{by \Cref{lm:dm}}
\end{align*}
    This implies $\distcal(\cD)\le O(\sqrt{\scdl(\cD)})$. 
    
    Consider an arbitrary joint distribution $\cJ$ of $(p_1,p_2,y)\in [0,1]\times [0,1]\times \{0,1\}$, where the distribution of $(p_1,y)$ is $\cD$ and the distribution $\cD'$ of $(p_2,y)$ is calibrated. We have $\scdl(\cD') = 0$, so by \Cref{thm:continuity},
    \[
    \scdl(\cD) = O(\sqrt{\E_\cJ|p_1 - p_2|}).
    \]
    This implies $\scdl(\cD) = O(\sqrt{\distcal(\cD)})$.
\end{proof}
\begin{remark}
\label{remark:tight}
    The upper and lower bounds on $\scdl$ in \Cref{thm:smce} are tight up to constant factors. For every $\varepsilon\in (0,1/4)$, Hu and Wu \cite{cdl} showed a distribution $\cD$ where $\distcal(\cD) = \varepsilon$ and $\cdl(\cD) = O(\varepsilon^2)$, which implies $\scdl(\cD) = O(\varepsilon^2)$ by \Cref{lm:cdl}. This shows that the lower bound of $\scdl$ in \Cref{thm:smce} is tight. For every $\varepsilon\in (0,1/4)$, the distribution $\cD_1$ in the proof of \Cref{thm:quadratic-smooth} has $\distcal(\cD_1) = O(\varepsilon)$ and $\scdl(\cD_1) = \Omega(\sqrt \varepsilon)$, proving the tightness of the upper bound.
\end{remark}

\section{Testability of SCDL} \label{sec:testability}
We formally state and prove the testability of $\scdl$.
\begin{theorem}[$\scdl$ is testable]
\label{thm:testable}
    Let $S$ be the uniform distribution over $T\ge 2$ data points $(p_1,y_1),\ldots,(p_T,y_T)$ drawn i.i.d.\ from a distribution $\cD$ on $[0,1]\times \{0,1\}$. For every $\delta\in (e^{-T/2},1/T)$, with probability at least $1-\delta$ we have
    \begin{equation}
    \label{eq:testable-1}
    |\scdl(S) - \scdl(\cD)| \le O\left(\sqrt{\frac{\log(1/\delta)}{T}} \log \frac{T}{\ln(1/\delta)}\right).
    \end{equation}
Moreover, when $\scdl(\cD) = 0$ the following improved bound holds: for every $\delta\in (e^{-T/2},1/T)$, with probability at least $1-\delta$ we have
\begin{equation}
\label{eq:testable-2}
\scdl(S) \le O\left(\sqrt{\frac{\log (1/\delta)}{T}\log\frac{T}{\ln(1/\delta)}}\right).
\end{equation}
\end{theorem}
Our proof of \Cref{thm:testable} relies on the following lemma:
\begin{lemma}
\label{lm:testable}
    Let $S$ be the uniform distribution over $T$ data points $(p_1,y_1),\ldots,(p_T,y_T)$ drawn i.i.d.\ from a distribution $\cD$ on $[0,1]\times \{0,1\}$. For every positive integer $m$ and every $i = 0,\ldots,m$, for every $\delta\in (0,1/(2m))$, with probability at least $1-\delta$ we have
    \begin{align*}
    & |\scfdl_{m,i}(S) - \scfdl_{m,i}(\cD)|\\
    \le {} & O\left( \sqrt{\frac{\log(1/\delta)}{T}} + (\log m)\sqrt{\frac {m\cdot\scdl_m(\cD)\log(1/\delta)}{T}} + \frac{m(\log m)\log(1/\delta)}{T}\right).
    \end{align*}
\end{lemma}

We prove \Cref{lm:testable} using \Cref{lm:testable-helper} below.
\begin{lemma} 
\label{lm:testable-helper}
Let $\cD$ be a distribution on $[0,1]\times \{0,1\}$. 
For every positive integer $m$, every $i = 0,\ldots,m$, define $\pi_i$ and $q_i$ as in \Cref{lm:equiv-def}. For every $\alpha > 0$, we have
    \begin{equation}
    \label{eq:testable-helper}
    \sum_{j < i} \ind[q_j > i/m -\alpha / \sqrt{\pi_j}] \sqrt {\pi_j} \le 3 +  O(\log m)(\alpha m + \sqrt{m\cdot \scdl_m(\cD)}).
    \end{equation}
\end{lemma}
\begin{proof}
When $m = 1$, the index $j$ can only take two possible values: $0$ and $1$. In this case, the left-hand side of \eqref{eq:testable-helper} is at most $2$, and the lemma holds trivially. We assume $m \ge 2$ from now on.

For $k = 0, 1, 2, 3,\ldots$, define $S_k:= \{j\in \{0,\ldots,m\}:2^{-k - 1}< \pi_j\le 2^{-k}\}$. Define
\begin{align*}
    i_k & := i - \lceil 2^{(k + 1)/2}\alpha m \rceil \le i - 1,\\
    S_k' & := \{j\in S_k: i_k  - 1\le j \le i - 3\},\\
    S_k'' & := \{j\in S_k: j \le i_k - 2, q_j > i_k / m\}.
\end{align*}

We prove the following two inequalities:
\begin{align}
    |S_k'| & \le O(2^{k/2}\alpha m), \label{eq:Sk-1}\\
    |S_k''| & \le O(2^{k/2}\sqrt{m\cdot \scdl_m(\cD)}).\label{eq:Sk-2}
\end{align}
By the definition of $S_k'$ and $i_k$, we have $|S_k'| \le i - i_k - 1 = \lceil 2^{(k + 1)/2}\alpha m\rceil - 1\le O(2^{k/2}\alpha m)$. This proves \eqref{eq:Sk-1}. Define $i_k' := i_k - \lceil |S_k''|  / 2\rceil - 1$. There are at least $|S_k''|/2$ different $j\in S_k''$ satisfying $j \le i_k'$. Therefore,
\[
\scdl_m(\cD) \ge \scfdl_{m,i_k'}(\cD) \ge 2^{-k - 1}(|S_k''|/2)(i_k/m - (i_k' + 1)/m)\ge 2^{-k - 1} (|S_k''|/2)^2 / m.
\]
Rearranging the inequality above proves \eqref{eq:Sk-2}. The lemma is proved by the following calculation:
\begin{align*}
    & \sum_{j < i} \ind[q_j > i/m -\alpha / \sqrt{\pi_j}]\sqrt{\pi_j}\\
    \le {} & \sqrt{\pi_{i - 1}} + \sqrt{\pi_{i - 2}} \\
    & \ \ + \sum_{j \le i, \pi_j \le 1/m^2}\sqrt{\pi_j} + \sum_{k = 0}^{\lceil \log_2(m^2) \rceil}\sum_{j \le i - 3}\ind[2^{-k - 1} < \pi_j \le 2^{-k}, q_j > i/m - \alpha/\sqrt{\pi_j}] \sqrt{\pi_j}\\
    \le {} & 1 + 1 + 1 + \sum_{k = 0}^{\lceil \log_2(m^2) \rceil}\sum_{j \le i - 3}\ind[2^{-k - 1} < \pi_j \le 2^{-k}, q_j > i_k/m] \sqrt{\pi_j}\\
    \le {} & 3 + \sum_{k = 0}^{\lceil \log_2(m^2) \rceil}\sum_{j \le i - 3}\ind[2^{-k - 1} < \pi_j \le 2^{-k}, q_j > i_k/m] \cdot 2^{-k/2} \tag{because $\pi_j \le 2^{-k}$}\\
    \le {} & 3 + \sum_{k = 0}^{\lceil \log_2(m^2)\rceil}2^{-k/2}|S_k'\cup S_k''|\\
    \le {} & 3 + O(\log m)(\alpha m + \sqrt{m\cdot \scdl_m(\cD)}).
\end{align*}

\end{proof}
\begin{proof}[Proof of \Cref{lm:testable}]
For every $j = 0,\ldots,m$, define $\pi_j$ and $q_j$ as in \Cref{lm:equiv-def}.
Define 
\begin{align*}
\hat \pi_j & := \E_{(p,y)\sim S}[w_i(p)] = \frac 1T\sum_{t = 1}^Tw_i(p_t)\\
\hat q_j & := \frac{1}{\hat \pi_j} \E_{(p,y)\sim S}[w_i(p)y] = \frac{1}{T\hat \pi_j}\sum_{t = 1}^Tw_i(p_t)y_t.
\end{align*}

By Bernstein's inequality and the union bound, with probability at least $1 - \delta$, for every $j = 0,\ldots,m$,
\begin{align}
|\hat \pi_j - \pi_j| & \le O\left(\sqrt {\frac{\pi_j \log(1/\delta)}{T}} + \frac{\log(1/\delta)}{T}\right),\label{eq:bernstein-1}\\
|\hat \pi_j \hat q_j - \pi_j q_j| & \le O\left(\sqrt {\frac{\pi_j \log(1/\delta)}{T}} + \frac{\log(1/\delta)}{T}\right). \label{eq:bernstein-2}
\end{align}
We assume that \eqref{eq:bernstein-1} and \eqref{eq:bernstein-2} hold.
For some absolute constant $C> 0$, as long as $\pi_j \ge \frac{C\log(1/\delta)}{T}$, we have
\begin{align}
|\hat \pi_j - \pi_j| & \le O\left(\sqrt {\frac{\pi_j \log(1/\delta)}{T}} \right),\label{eq:bernstein-3}\\
|\hat \pi_j \hat q_j - \pi_j q_j| & \le O\left(\sqrt {\frac{\pi_j \log(1/\delta)}{T}} \right), \quad \text{and}\label{eq:bernstein-4}\\
|\hat q_j - q_j| & \le \frac 1{\pi_j}(|\hat \pi_j \hat q_j - \pi_j q_j| + |\hat \pi_j - \pi_j|\hat q_j) \le O\left(\sqrt{\frac{\log(1/\delta)}{T\pi_j}}\right).\label{eq:bernstein-5}
\end{align}
The last inequality implies $|\hat q_j - q_j| \le \alpha/\sqrt {\pi_j}$ for some $\alpha = O\left(\sqrt{\frac{\log(1/\delta)}T}\right)$.
    
    Define $H_-$ as the set of indices $j\in \{0,\ldots,i\}$ satisfying $\max\{q_j,\hat q_j\} > (i + 1)/m$. Define $H_+$ as the set of indices $j\in \{i + 1,\ldots,m\}$ satisfying $\min\{q_j,\hat q_j\} < i/m$.
    By the triangle inequality,
    \begin{align}
        & |\scfdl_{m,i}(S) - \scfdl_{m,i}(\cD)|\notag \\
        \le {} & \sum_{j\le i} |\hat \pi_j(\hat q_j - (i + 1)/m)_+ - \pi_j(q_j - (i + 1)/m)_+| + \sum_{j> i}|\hat \pi_j(i/m - \hat q_j)_+ - \pi_j(i/m - q_j)_+|\notag \\
        \le {} & \sum_{j\in H_-}|\hat \pi_j(\hat q_j - (i+1)/m) - \pi_j(q_j - (i + 1)/m)| + \sum_{j \in H_+}|\hat \pi_j (i/m - \hat q_j) - \pi_j(i/m - q_j)| \notag \\
        \le {} & \sum_{j\in H_-}|\hat \pi_j\hat q_j - \pi_j q_j| + \sum_{j\in H_-}|\pi_j - \hat \pi_j| + \sum_{j\in H_+}\pi_j |q_j - \hat q_j| + \sum_{j\in H_+}|\pi_j - \hat \pi_j|.\label{eq:scdlm}
    \end{align}
Define $L$ as the set of indices $j = 0,\ldots,m$ satisfying $\pi_j < \frac{C\log(1/\delta)}{T}$. By \eqref{eq:bernstein-3}, we have $\hat \pi_j = O(\frac{\log(1/\delta)}{T})$. Therefore,
\begin{equation}
\label{eq:scdlm-part-1}
\sum_{j\in L}|\hat \pi_j \hat q_j - \pi_j q_j| + \sum_{j\in L}|\pi_j - \hat \pi_j| \le 2|L| \cdot O\left(\frac{\log(1/\delta)}{T}\right) = O\left(\frac{m\log(1/\delta)}{T}\right).
\end{equation}
Moreover, by \eqref{eq:bernstein-4} and \eqref{eq:bernstein-5},
\begin{align}
    & \sum_{j\in H_-\setminus L}|\hat \pi_j \hat q_j - \pi_j q_j| + \sum_{j\in H_- \setminus L} |\hat \pi_j - \pi_j|\notag \\
    \le {} & O\left(\sqrt\frac{\log(1/\delta)}{T}\right)\left(\sum_{j\in H_-\setminus L}\sqrt {\pi_j}\right)\notag \\
    \le {} & O\left(\sqrt\frac{\log(1/\delta)}{T}\right)\cdot (3 +  O(\log m)(\alpha m + \sqrt{m\cdot \scdl_m(\cD)}))\tag{by \Cref{lm:testable-helper}} \\
    = {} & O\left( \sqrt{\frac{\log(1/\delta)}{T}} + (\log m)\sqrt{\frac {m\cdot\scdl_m(\cD)\log(1/\delta)}{T}} + \frac{m(\log m)\log(1/\delta)}{T}\right).\label{eq:scdlm-part-2}
\end{align}
Similarly, we have
\begin{align}
& \sum_{j\in H_+\setminus L}|\hat \pi_j \hat q_j - \pi_j q_j| + \sum_{j\in H_+ \setminus L} |\hat \pi_j - \pi_j|\notag\\
= {} & O\left( \sqrt{\frac{\log(1/\delta)}{T}} + (\log m)\sqrt{\frac {m\cdot\scdl_m(\cD)\log(1/\delta)}{T}} + \frac{m(\log m)\log(1/\delta)}{T}\right).\label{eq:scdlm-part-3}
\end{align}
Combining \eqref{eq:scdlm-part-1}, \eqref{eq:scdlm-part-2}, and \eqref{eq:scdlm-part-3} and plugging the result into \eqref{eq:scdlm} completes the proof.
\end{proof}

The same proof of \Cref{lm:testable} establishes the following lemma, which we used in \Cref{sec:actionability} to prove \Cref{thm:actionability-hp}. 
\begin{lemma}
\label{lm:hp-helper}
        Let $(p_1,y_1),\ldots,(p_T,y_T)\in [0,1]\times \{0,1\}$ be an arbitrary sequence of prediction-outcome pairs and let $\cD$ be the uniform distribution over these pairs.  Let $m$ be a positive integer. For every $t = 1,\ldots,T$, we independently round $p_i$ to $p_i'\in G_m$ using the randomized rounding function $\tau_m$ in \Cref{def:tau} and let $\cD'$ be the uniform distribution over $(p_1',y_1),\ldots,(p_T',y_T)$. Then for every $\delta\in (0,1/(2m))$, with probability at least $1-\delta$ we have
    \begin{align*}
    & |\scdl_{m}(\cD') - \scdl_{m}(\cD)|\\
    \le {} & O\left( \sqrt{\frac{\log(1/\delta)}{T}} + (\log m)\sqrt{\frac {m\cdot\scdl_m(\cD)\log(1/\delta)}{T}} + \frac{m(\log m)\log(1/\delta)}{T}\right).
    \end{align*}
\end{lemma}

\begin{proof}[Proof of \Cref{thm:testable}]
It holds trivially that $|\scdl(S) - \scdl(\cD)| \le 1$. We can thus assume without loss of generality that $T/\log(1/\delta) \ge 100 C_0$ for an absolute constant $C_0 > 0$ which we will determine later. Let $m_{\max}$ be the largest number in $M:= \{2^1,2^2,\ldots,\}$ such that $m_{\max}\log m_{\max} \le \sqrt{\frac{T}{C_0\log(1/\delta)}}$.

Define $m^*:= m^*(\cD)\in M\cup \{+\infty\}$. We divide the analysis into two cases: $m^* < m_{\max}$ and $m^* \ge m_{\max}$.

When $m^* < m_{\max}$, by \Cref{lm:testable} and the union bound, with probability at least $1-\delta$ the following two conditions hold simultaneously:
\begin{align}
    & |\scdl_{m^*}(S) - \scdl_{m^*}(\cD)| \notag \\
    \le {} & O\left( \sqrt{\frac{\log(1/\delta)}{T}} + (\log m^*)\sqrt{\frac {m^*\cdot\scdl_{m^*}(\cD)\log(1/\delta)}{T}} + \frac{m^*(\log m^*)\log(1/\delta)}{T}\right)\notag \\
    \le {} & O\left(  (\log m^*)\sqrt{\frac {\log(1/\delta)}{T}} + \frac{m^*(\log m^*)\log(1/\delta)}{T}\right), \label{eq:m*}
\end{align}
and
\begin{align}
    & |\scdl_{2m^*}(S) - \scdl_{2m^*}(\cD)| \notag \\ \le {} & O\left( \sqrt{\frac{\log(1/\delta)}{T}} + (\log m^*)\sqrt{\frac {m^*\cdot\scdl_{2m^*}(\cD)\log(1/\delta)}{T}} + \frac{m^*(\log m^*)\log(1/\delta)}{T}\right) \notag \\
    \le {} & O\left( (\log m^*)\sqrt{\frac {m^*\cdot\scdl_{2m^*}(\cD)\log(1/\delta)}{T}} + \frac{m^*(\log m^*)\log(1/\delta)}{T}\right), \label{eq:2m*}
\end{align}
where we used $m^*\cdot \scdl_{m^*}(\cD) < 2$ and $m^*\cdot \scdl_{2m^*}(\cD) \ge 1$ implied by \Cref{lm:m*} to simplify the expressions. We assume \eqref{eq:m*} and \eqref{eq:2m*} hold from now on.

By \eqref{eq:2m*}, there exists an absolute constant $C > 0$ such that
\begin{align}
& \scdl_{2m^*}(S) \notag \\
\ge {} & \scdl_{2m^*}(\cD) - C\left( (\log m^*)\sqrt{\frac {m^*\cdot\scdl_{2m^*}(\cD)\log(1/\delta)}{T}} + \frac{m^*(\log m^*)\log(1/\delta)}{T}\right).\label{eq:2m*-1}
\end{align}
By \Cref{lm:m*}, $\scdl_{2m^*}(\cD) \ge \scdl(\cD) \ge 1/m^*$. By our choice of $m_{\max}$, we have
\begin{equation}
\label{eq:m*-mmax}
m^*\log m^* \le m_{\max} \log m_{\max} \le \sqrt{\frac{T}{C_0\log(1/\delta)}}.
\end{equation}
Choosing $C_0$ to be a sufficiently large constant, we can ensure that the right-hand side of \eqref{eq:2m*-1} is an increasing function of $\scdl_{2m^*}(\cD)\in [1/m^*,1]$. Therefore, \eqref{eq:2m*-1} implies
\begin{align}
& \scdl_{2m^*}(S) \notag \\
\ge {} & \scdl(\cD) - C\left( (\log m^*)\sqrt{\frac {m^*\cdot\scdl(\cD)\log(1/\delta)}{T}} + \frac{m^*(\log m^*)\log(1/\delta)}{T}\right)\notag \\
\ge {} & \scdl(\cD) - O\left((\log m^*)\sqrt{\frac {\log(1/\delta)}{T}} + \frac{m^*(\log m^*)\ln(1/\delta)}{T}\right),\label{eq:2m*-2}
\end{align}
where the last inequality in \eqref{eq:2m*-2} holds because $m^*\cdot \scdl(\cD) < 2$.

Let us establish an upper bound on $\scdl(S)$:
\begin{align*}
\scdl(S) & \le \max\{\scdl_{m^*}(S), 1/m^*\}\\
& \le \max\{\scdl_{m^*}(\cD), 1/m^*\} + |\scdl_{m^*}(S) - \scdl_{m^*}(\cD)|\\
& = \scdl(\cD) + |\scdl_{m^*}(S) - \scdl_{m^*}(\cD)|.
\end{align*}
Plugging \eqref{eq:m*} into the inequality above, we get
\begin{equation}
\label{eq:scdl-upper}
\scdl(S) \le \scdl(\cD) + O\left( (\log m^*)\sqrt{\frac {\log(1/\delta)}{T}} + \frac{m^*(\log m^*)\log(1/\delta)}{T}\right).
\end{equation}
We now turn to proving a lower bound for $\scdl(S)$:
\begin{align}
\scdl(S) & = \min \Big\{\inf_{m < m^*}\max\{\scdl_{m}(S),1/m\},\notag \\
& \quad \quad\quad\quad\max\{\scdl_{m^*}(S), 1/m^*\},\notag \\
& \quad \quad\quad \quad\inf_{m > m^*}\max\{\scdl_m(S),1/m\}\Big\} \notag \\
& \ge \min\{2/m^*, \max\{\scdl_{m^*}(S),1/m^*\}, \scdl_{2m^*}(S)\}.\label{eq:scdl-lower-0}
\end{align}
The first two terms in the minimum above can be lower-bounded as follows using \Cref{lm:m*}:
\begin{align*}
    2/m^* & > \scdl(\cD),\\
    \max\{\scdl_{m^*}(S),1/m^*\} & \ge \max\{\scdl_{m^*}(\cD),1/m^*\} - |\scdl_{m^*}(S) - \scdl_{m^*}(\cD)|\\
    & = \scdl(\cD) - |\scdl_{m^*}(S) - \scdl_{m^*}(\cD)|.
\end{align*}
The third term $\scdl_{2m^*}(S)$ can be lower-bounded by \eqref{eq:2m*-2}. Plugging these into \eqref{eq:scdl-lower-0} and using \eqref{eq:m*}, we get
\begin{equation}
\label{eq:scdl-lower}
\scdl(S) \ge \scdl(\cD) - O\left( (\log m^*)\sqrt{\frac {\log(1/\delta)}{T}} + \frac{m^*(\log m^*)\log(1/\delta)}{T}\right).
\end{equation}
Combining \eqref{eq:scdl-upper} and \eqref{eq:scdl-lower}, we get
\[
|\scdl(S) - \scdl(\cD)| = O\left((\log m^*)\sqrt{\frac {\log(1/\delta)}{T}} + \frac{m^*(\log m^*)\log(1/\delta)}{T}\right).
\]
Combining this with \eqref{eq:m*-mmax} proves \eqref{eq:testable-1}.

When $m^* \ge m_{\max}$, \Cref{lm:m*} implies 
\begin{equation}
\label{eq:m-max-0}
\scdl_{m_{\max}}(\cD) \le \scdl_{m^*}(\cD) \le \scdl(\cD) < 2/m_{\max}. 
\end{equation}
By \Cref{lm:testable} and the union bound, with probability at least $1-\delta$ it holds that
\begin{equation}
\label{eq:m-max}
|\scdl_{m_{\max}}(S) - \scdl_{m_{\max}}(\cD)| = O \left((\log {m_{\max}})\sqrt{\frac {\log(1/\delta)}{T}} + \frac{m_{\max}(\log m_{\max})\log(1/\delta)}{T}\right).
\end{equation}
By \eqref{eq:m-max-0},
\begin{align}
\scdl(S) \le {} & \max\{\scdl_{m_{\max}}(S),1/m_{\max}\}\notag \\
\le {} & \max\{\scdl_{m_{\max}}(\cD),1/m_{\max}\} + |\scdl_{m_{\max}}(S) - \scdl_{m_{\max}}(\cD)|\notag \\
\le {} & 2/m_{\max} + |\scdl_{m_{\max}}(S) - \scdl_{m_{\max}}(\cD)|.\label{eq:m-max-2}
\end{align}
Plugging \eqref{eq:m-max} into \eqref{eq:m-max-2} and using $\scdl(\cD) < 2/m_{\max}$, we have
\[
|\scdl(S) - \scdl(\cD)| \le O\left(\frac{1}{m_{\max}} + (\log {m_{\max}})\sqrt{\frac {\log(1/\delta)}{T}} + \frac{m_{\max}(\log m_{\max})\log(1/\delta)}{T} \right).
\]
This proves \eqref{eq:testable-1} by our choice of $m_{\max}$.

When $\scdl(\cD) = 0$, we choose $m_{\max}$ as the largest number in $M$ such that
\[
m_{\max} \sqrt{\log m_{\max}} \le \sqrt{\frac{T}{ C_0 \log(1/\delta)}}.
\]
By \Cref{lm:testable} and the union bound, with probability at least $1-\delta$ we have
\[
\scdl_{m_{\max}}(S) = O\left(\sqrt{\frac{\log(1/\delta)}{T}} + \frac{m_{\max}(\log m_{\max})\log(1/\delta)}{T}\right).
\]
Therefore,
\begin{align*}
    \scdl(S) & \le \max\{\scdl_{m_{\max}}(S), 1/m_{\max}\}\\
    & = O\left(\frac{1}{m_{\max}} + \sqrt{\frac{\log(1/\delta)}{T}} + \frac{m_{\max}(\log m_{\max})\log(1/\delta)}{T}\right).
\end{align*}
This proves \eqref{eq:testable-2} by our choice of $m_{\max}$.
\end{proof}

\section{Lower bound}
\label{sec:lb}
In this section, we prove an $\Omega(T^{-1/2})$ lower bound for the estimation error $\varepsilon(T)$ in \Cref{question}, showing that our upper bound is tight up to a logarithmic factor.
\begin{theorem}
\label{thm:lb}
Any calibration measure $\CAL_T$ solving \Cref{question} requires $\varepsilon(T) \ge \Omega(T^{-1/2})$.
\end{theorem}
We use the following lemma to prove \Cref{thm:lb}.
\begin{lemma}
\label{lm:lb}
    There exists a perfectly calibrated distribution $\cD$ of $(p,y)\in [0,1]\times \{0,1\}$ and a decision task $Z = (A,U)$ with the following properties:
    \begin{enumerate}
        \item The utility function $U$ is bounded between $0$ and $1$, i.e.,  $U:A\times \{0,1\}\to [0,1]$.
        \item Let $S$ be the uniform distribution on $T$ i.i.d.\ data points drawn from $\cD$. For every positive integer $T$ and every randomized response function $r:[0,1]\to \Delta(A)$, with probability at least $1/3$ over the randomness in $S$,
        \[
        \sr_Z(r,S) = \Omega(T^{-1/2}).
        \]
    \end{enumerate}
\end{lemma}
\begin{proof}
    Let $\cD$ be the uniform distribution on $(1/2,0)$ and $(1/2,1)$. Let $U$ be the simple decision task with binary actions $A = \{0,1\}$ and a 0-1 utility function $U(a,y) = \ind[a = y]$. Consider the action distribution $r(1/2)$ produced by an arbitrary randomized response function $r$. Since there are only two actions $a = 0,1$, $r(1/2)$ must put probability mass at least $1/2$ on one of the two actions. We assume without loss of generality that $r(1/2)$ puts at least probability $1/2$ on action $a = 0$.

    On a random sample of $T$ i.i.d.\ points $(p_1,y_1),\ldots,(p_T,y_T)$ drawn from $\cD$, we have $p_1 = \cdots = p_T = 1/2$. Also, $y_1,\ldots,y_T$ are drawn independently from $\Ber(1/2)$, so with probability at least $1/3$, we have 
    \begin{equation}
    \label{eq:unbiased-coin}
    \frac 1T \sum_{t = 1}^Ty_t \ge 1/2 + \Omega(T^{-1/2}).
    \end{equation}
It remains to show that whenever \eqref{eq:unbiased-coin} holds, we have $\sr_Z(r,S) = \Omega(T^{-1/2})$. When \eqref{eq:unbiased-coin} holds, for $(p,y)$ drawn from  $S$, we have $p = 1/2$ deterministically and $y$ is drawn from a Bernoulli distribution with mean $1/2 + \Omega(T^{-1/2})$. For this distribution of $y$, the best action to take is $a = 1$. A probability $q$ of taking the wrong action $a = 0$ incurs swap regret $\Omega(qT^{-1/2})$. Since $r(1/2)$ puts probability at least $1/2$ on action $a = 0$, it incurs swap regret $\Omega(T^{-1/2})$.
\end{proof}
\begin{proof}[Proof of \Cref{thm:lb}]
    Let $\CAL_T$ be a solution to \Cref{question}. Consider the distribution $\cD$ and the decision task $Z$ from \Cref{lm:lb}. The second requirement of \Cref{question} implies $\E[\CAL_T(S_T)]\le \varepsilon(T)$, where $S_T$ is the uniform distribution over $(p_1,y_1),\ldots,(p_T,y_T)$ drawn i.i.d.\ from $\cD$. By Markov's inequality, $\Pr[\CAL_T(S_T) \le 4\varepsilon(T)] \ge 3/4$. Setting $\alpha = 4\varepsilon(T)$ in the first requirement of \Cref{question}, we have $\Pr[\sr_Z(r_{Z,\alpha},
    S_T) \le 4\varepsilon(T)] \ge 3/4$. By \Cref{lm:lb}, $\Pr[\sr_Z(r_{Z,\alpha},S_T) \ge \Omega(T^{-1/2})] \ge 1/3$. Combining the last two inequalities, we get $\varepsilon(T) \ge \Omega(T^{-1/2})$.
\end{proof}

\section{Experiments}
\label{sec:experiments}
\label{sec:add_exp}
We follow the experimental setup of \citet{test-action} and extend it with additional experiments.\footnote{Code for our experiments is available at [\url{https://github.com/konstantinabrk/scdl}].} For a given $\alpha \in [0,1]$, each sample $(X_i,Y_i)$ is drawn as follows: the feature $X_i$ is sampled independently from the uniform distribution on $[0,1]$, and the label $Y_i$ is drawn from a Bernoulli distribution with mean 
$\E[Y|X=x] = \alpha (1-2x)^2+(1-\alpha)x$. We fit a predictor $f$ using logistic regression on $n=500$ samples from the distribution. To break monotonicity, we also consider the predictor that flips the labels, i.e. $1-f$. 

To build intuition of the data distribution and the predictors used in our experiments, we plot $\E[Y|f(x)], f(x)$ and $1-f(x)$ for different values of $\alpha$ in \Cref{fig:calibration_curves}.

\begin{figure}[ht]
\centering
\begin{subfigure}{0.32\textwidth}
    \includegraphics[width=\textwidth]{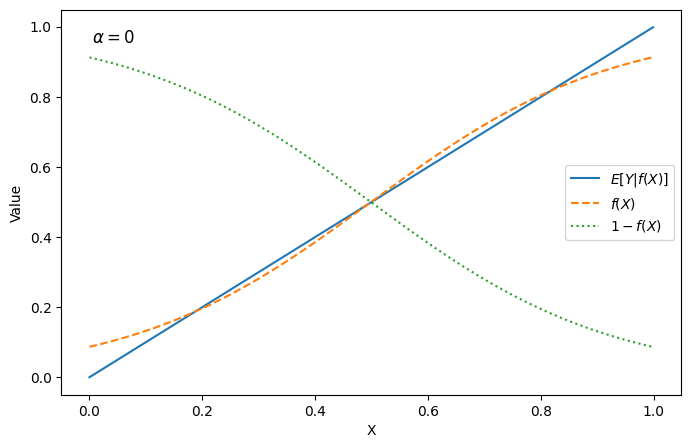}
    \caption{$\alpha = 0.0$}
\end{subfigure}
\begin{subfigure}{0.32\textwidth}
    \includegraphics[width=\textwidth]{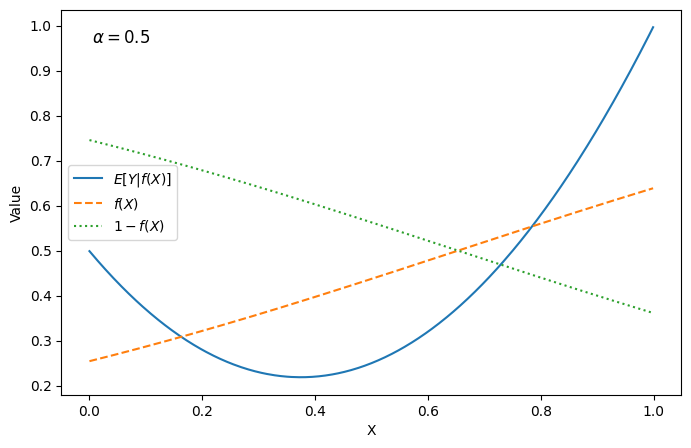}
    \caption{$\alpha = 0.5$}
\end{subfigure}
\\[1em]
\begin{subfigure}{0.32\textwidth}
    \includegraphics[width=\textwidth]{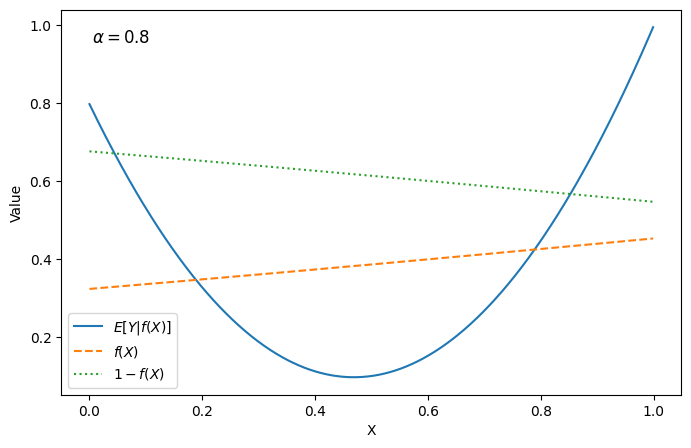}
    \caption{$\alpha = 0.8$}
\end{subfigure}
\begin{subfigure}{0.32\textwidth}
    \includegraphics[width=\textwidth]{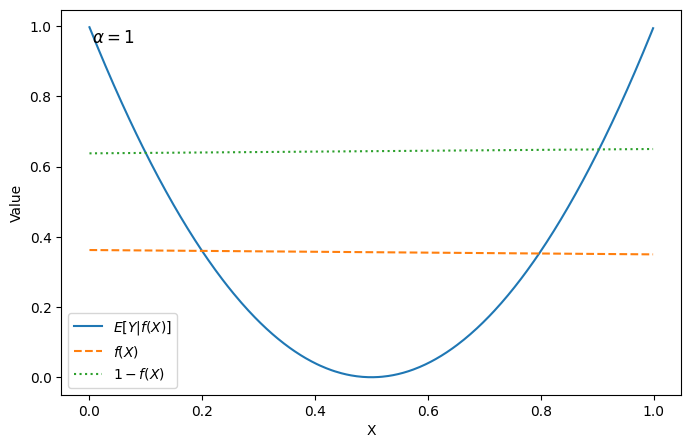}
    \caption{$\alpha = 1$}
\end{subfigure}
\caption{Prediction curves for different values of $\alpha$.}
\label{fig:calibration_curves}
\end{figure}

\paragraph{Testability.}
In \Cref{tab:calibration_metrics}, we compute the value and the standard deviation of the estimated calibration errors from $N=1000$ samples. To estimate $\ECE$, we use $\bECE$ for $10$ bins of equal length. 

\begin{table}[ht]
\caption{Calibration error estimates (mean $\pm$ std) for predictor $f$ across values of $\alpha$.}
\label{tab:calibration_metrics}
\centering
\begin{tabular}{lcccc}
\toprule
Measure & $\alpha = 0$ & $\alpha = 0.5$ & $\alpha = 0.8$ & $\alpha = 1$ \\
\midrule
$\smCE$   & $0.021 \pm 0.014$ & $0.028 \pm 0.013$ & $0.027 \pm 0.016$ & $0.025 \pm 0.016$ \\
$\mathsf{Cutoff}$ & $0.030 \pm 0.012$ & $0.068 \pm 0.016$ & $0.110 \pm 0.016$ & $0.136 \pm 0.015$ \\
$\mathsf{bECE}$    & $0.043 \pm 0.011$ & $0.117 \pm 0.015$ & $0.140 \pm 0.054$ & $0.064 \pm 0.065$ \\
$\scdl$   & $0.016 \pm 0.003$ & $0.036 \pm 0.006$ & $0.080 \pm 0.014$ & $0.076 \pm 0.034$ \\
\bottomrule
\end{tabular}

\end{table}

Across most values of $\alpha$ in \Cref{tab:calibration_metrics}, $\scdl$ has the smallest standard deviation among the four error measures, suggesting it can be estimated more reliably from finite samples. The exception is $\alpha=1$, where $\smCE$ and $\mathsf{Cutoff}$ have lower variance, though $\scdl$ still compares favorably to $\mathsf{bECE}$.

\paragraph{Actionability.}

We evaluate actionability using the threshold utility from \citet{test-action}, with each probability treated as a distinct action: the action space is $A = [0,1]$, the utility function is $U(a,y) = 1-0.35(1-y)\ind[a \geq 0.35]-0.65y\ind[a < 0.35]$, and we choose the best-response function $r_Z^*(p) = p$.
Although actions on the same side of the threshold have identical utilities, a swap function $\sigma:A\to A$ may map them to different sides of the threshold. The swap-regret benchmark therefore chooses the better of the two utility types separately for each prediction.
For $200$ uniformly random choices of $\alpha$, we compute each calibration measure and the swap regret of the best-response action using the same setup as above. To keep the comparison fair across measures, we compute swap regret using the original, non-discretized best-response function $r_Z^*$ for every measure, including $\scdl$. This is a more demanding test for $\scdl$. Our actionability guarantee holds for the discretized response function $r_{Z,\alpha}$, so this experiment checks whether $\scdl$ remains informative even when the decision maker does not discretize the predictions.
We show scatter plots of the results for $f$ in \Cref{fig:correlation_f_200} and for $1-f$ in \Cref{fig:correlation_1-f_200}, where darker shades of blue correspond to higher $\alpha$.%
Following \citet{test-action}, we quantify the alignment between each calibration error measure and the swap regret using the Spearman correlation, where a higher correlation means that the calibration error more faithfully informs the decision maker about their swap regret. While actionability is a one-sided guarantee, upper bounding the swap regret simultaneously over all bounded decision tasks, an individual decision maker would also like the calibration error to indicate the swap regret on their own task. While our theoretical results establish the upper bound, this experiment complements them by measuring, on a single fixed decision task as in \citet{test-action}, how well each calibration error tracks the swap regret.

\begin{figure}[htbp]
  \centering
  \includegraphics[width=0.8\textwidth]{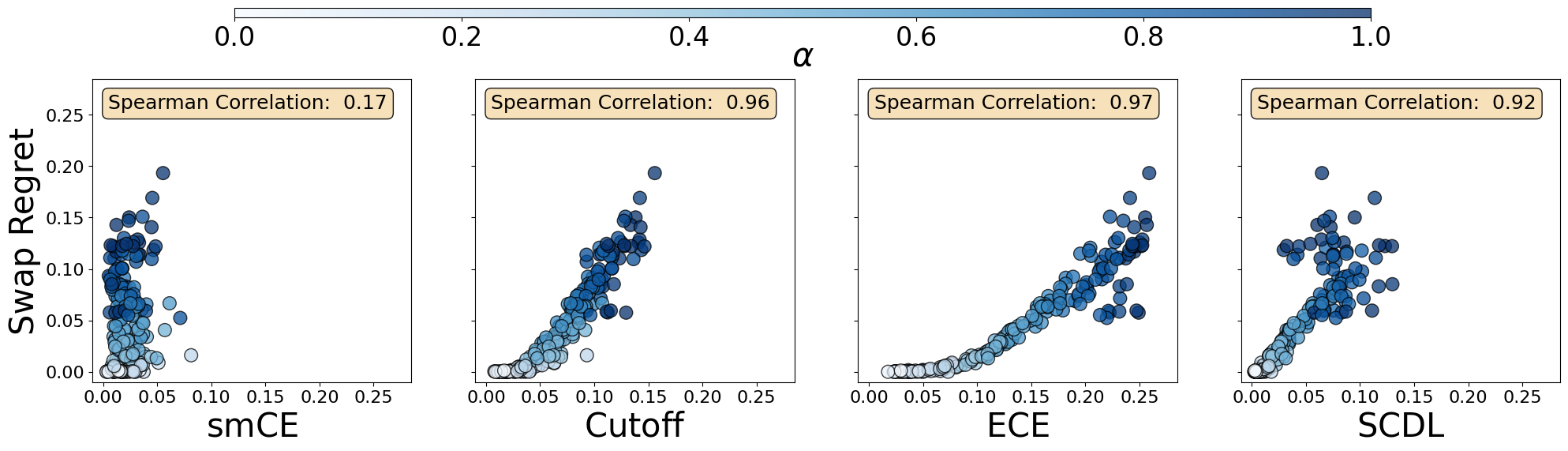}
  \caption{Plots of calibration error measures versus swap regret for predictor $f$. Spearman correlations are reported in each subplot.}
  \label{fig:correlation_f_200}
\end{figure}

\begin{figure}[htbp]
  \centering
  \includegraphics[width=0.8\textwidth]{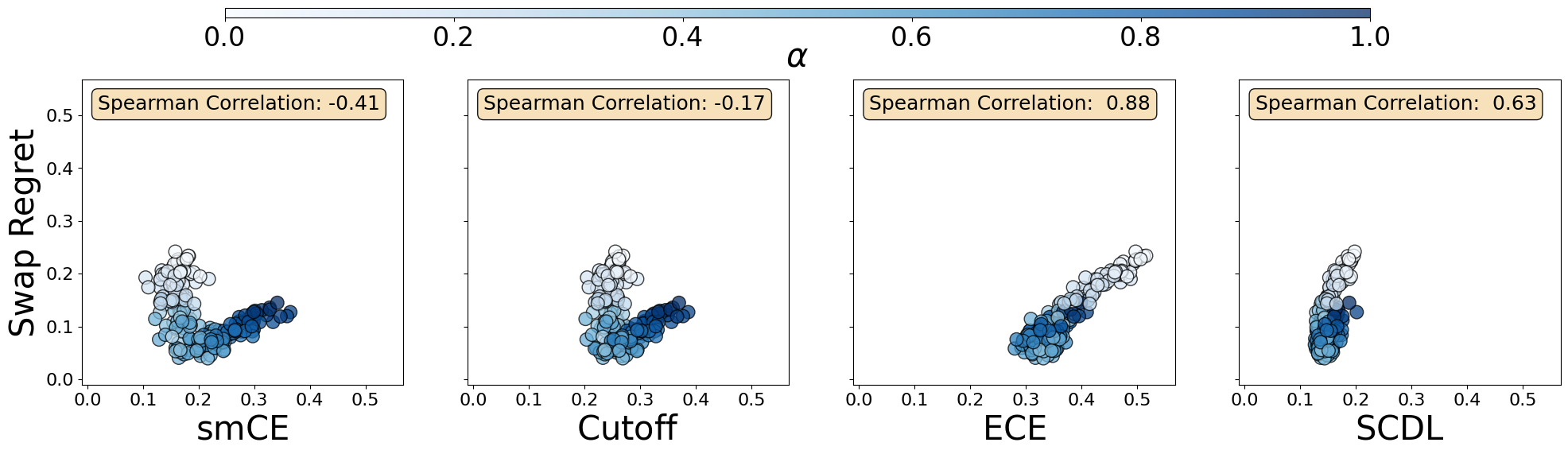}
  \caption{Plots of calibration error measures versus swap regret for predictor $1-f$. Spearman correlations are reported in each subplot.}
  \label{fig:correlation_1-f_200}
\end{figure}

For predictor $f$, all measures except $\smCE$ are highly correlated with swap regret (see \Cref{fig:correlation_f_200}). However, the picture changes for $1-f$. $\mathsf{Cutoff}$ is no longer correlated with swap regret, it just upper bounds it, with several (prediction, outcome) distributions sharing similar swap regret but different $\mathsf{Cutoff}$ values. $\ECE$ remains highly correlated and $\scdl$ shows moderate correlation, indicating that it tracks swap regret adequately well. This can be explained by the fact that the optimal swap function for $1-f$ is not monotone, falling outside the swap function class that $\mathsf{Cutoff}$ restricts to when bounding the swap regret. These results support our approach of relaxing actionability by discretizing the predictions rather than by restricting the swap functions.

\section*{Acknowledgments}
The authors thank Yifan Wu for insightful discussions that initiated this work as well as helpful conversations about \cite{smooth-decision}. We also thank Parikshit Gopalan for conversations about \cite{efficient-decision,importance}.

%% file: appendix.tex
\section{Quadratic Gap between Smooth Calibration Error and Swap Regret}
\label{sec:lb-smooth}
\label{sec:quadratic-smooth}
\begin{theorem}[\cite{smooth-decision}]
\label{thm:quadratic-smooth}
    For every $\varepsilon\in [0,1/8)$ there exist two distributions $\cD_1$ and $\cD_2$ of $(p,y)\in [0,1]\times \{0,1\}$ and a decision task $Z = (A,U)$ such that 
    \begin{enumerate}
        \item $\smCE(\cD_1) = \smCE(\cD_2) = \varepsilon$,
        \item the utility function $U$ is bounded between $0$ and $1$, i.e.,  $U:A\times \{0,1\}\to [0,1]$, and
        \item    for every randomized response function $r:[0,1]\to \Delta(A)$,
    \[
    \max\{\sr_Z(r,\cD_1), \sr_Z(r,\cD_2)\} \ge \Omega(\sqrt \varepsilon).
    \]
    \end{enumerate}
\end{theorem}
\begin{proof}
    Let $\cD_1$ be the distribution of $(p,y)\in [0,1]\times \{0,1\}$ where $p$ is uniform from $\{1/2 - \sqrt \varepsilon, 1/2 + \sqrt \varepsilon\}$ and $y$ is drawn independently and uniformly from $\{0,1\}$. Let $\cD_2$ be the distribution of $(p,y)$ where $p = 1/2 - \sqrt \varepsilon$ deterministically and $y\sim \Ber(1/2 - \sqrt \varepsilon - c\varepsilon)$ for a parameter $c > 0$. Any $c = O(1)$ ensures that both $\cD_1$ and $\cD_2$ have smooth calibration error $O(\varepsilon)$, and we can choose $c$ appropriately to make the two distributions have the same smooth calibration error. Consider a simple decision task with binary actions $A = \{0,1\}$ and utility function
\[
u(a,y) = (1 - \sqrt \varepsilon)\ind[a = 0,y = 0] + \ind[a = 1,y = 1].
\]
On $\cD_1$, the optimal action to take is always $a = 1$, regardless of the prediction $p$. On $\cD_2$, the optimal action to take is always $a = 0$. A probability $q$ of taking the wrong action incurs $\Omega(q\sqrt \varepsilon)$ expected swap regret, and in fact $\Omega(q\sqrt \varepsilon)$ external regret as well. Since both $\cD_1$ and $\cD_2$ put a probability mass of at least $1/2$ on the same prediction $1/2 - \sqrt \varepsilon$, any post-processing, even a randomized one, will lead to a probability of at least $1/4$ taking the wrong action on either $\cD_1$ or $\cD_2$.
\end{proof}

\section{Power-2/3 Gap between Cutoff Calibration Error and Swap Regret}
\label{sec:lb-cutoff}
\label{sec:gap-cutoff}
\begin{theorem}
\label{thm:gap-cutoff}
    Let $Z = (A,U)$ be the decision task where $A = [0,1]$ and $U(a,y) = 1 - (a - y)^2$. For every $\varepsilon \in (0,10^{-100})$ and every randomized response function $r:[0,1]\to \Delta(A)$, there exists a distribution $\cD$ with $\cutoff(\cD)  = \varepsilon$ and $\sr_Z(r,\cD) \ge \Omega(\varepsilon^{2/3})$.
\end{theorem}
The proof of the this theorem as well as the exponent $2/3$ was discovered by ChatGPT Pro. We independently verified that the proof is correct and present it below. 
We first prove the following helper lemma:
\begin{lemma}
\label{lm:cutoff-helper}
    Let $Z$ be the decision task from \Cref{thm:gap-cutoff}. 
    Let $\cD$ be a distribution on $[0,1]\times \{0,1\}$ and let $r:[0,1]\to \Delta([0,1])$ be a randomized response function.
    Let $\Gamma$ be the distribution of $(a,y)\in [0,1]\times \{0,1\}$ where we first draw $(p,y)$ from $\cD$ and then sample $a\sim r(p)$. We have
    \[
    \sr_Z(r,\cD) = \E_\Gamma[(a - \E_\Gamma[y|a])^2].
    \]
    Moreover, for every function $\varphi:[0,1]\to \R$, it holds that
    \begin{equation}
    \label{eq:varphi}
    \sr_Z(r,\cD) \E_\Gamma[\varphi(a)^2] \ge \E_\Gamma[(y - a)\varphi(a)]^2.
    \end{equation}
\end{lemma}
\begin{proof}
    By the definition of swap regret,
    \begin{equation}
    \label{eq:cutoff-swap-regret}
            \sr_Z(r,\cD) = \sup_{\sigma:[0,1]\to [0,1]}\E_{(a,y)\sim \Gamma}[(a - y)^2 - (\sigma(a) - y)^2].
    \end{equation}
    We can simply the expression using the following calculation:
    \begin{align*}
    \E[(a - y)^2] & = \E[(a - \E[y|a])^2] + \E[(y - \E[y|a])^2] - 2 \, \E[(a - \E[y|a])(y - \E[y|a])]\\
    & = \E[(a - \E[y|a])^2] + \E[(y - \E[y|a])^2].
    \end{align*}
    Similarly, $\E[(\sigma(a) - y)^2] = \E[(\sigma(a) - \E[y|a])^2] + \E[(y - \E[y|a])^2]$. Therefore, the supremum in \eqref{eq:cutoff-swap-regret} is achieved when $\sigma(a) = \E[y|a]$. Now \eqref{eq:cutoff-swap-regret} simplifies to
    \[
    \sr_Z(r,\cD) = \E_{(a,y)\sim \Gamma}[(a - \E[y|a])^2].
    \]
    By the Cauchy-Schwarz inequality,
    \[
    \sr_Z(r,\cD) \E[\varphi(a)^2] \ge \E[(a - \E[y|a])\varphi(a)]^2 = \E[(a - y)\varphi(a)]^2. \qedhere
    \]
\end{proof}

\begin{proof}[Proof of \Cref{thm:gap-cutoff}]
    Consider an arbitrary distribution $\cD$ on $[0,1]\times \{0,1\}$. Let $\Gamma$ be the distribution of $(a,y)$, where we first draw $(p,y)\sim \cD$ and then draw $a\sim r(p)$. 
    By \Cref{lm:cutoff-helper},
    \[
    \sr_Z(r,\cD) = \E_\Gamma[(a - \E[y|a])^2].
    \]

Let $\cD$ be the distribution of $(p,y)$ where $p = p_0\in [0,1-\varepsilon]$ is deterministic and $\E[y] = p_0 + \varepsilon$. Clearly, we have $\cutoff(\cD) = \varepsilon$. In this case, $y$ and $a$ are independent, so we have $\E[y|a] = p_0 + \varepsilon$ and
\[
\sr_Z(r,\cD) = \E_{a\sim r(p_0)}[(a - (p_0 + \varepsilon))^2] \ge \E_{a\sim r(p_0)}[(a - p_0)^2] - 2\varepsilon.
\]
If $\E_{a\sim r(p_0)}[(a - p_0)^2] > 10^{-6}\varepsilon^{2/3}$, then we have $\sr_Z(r,\cD) \ge \Omega(\varepsilon^{2/3})$ and the theorem holds. 

We can thus assume that $\E_{a\sim r(p)}[(a - p)^2] \le 10^{-6}\varepsilon^{2/3}$ for every $p\in [0,1-\varepsilon]$. By Markov's inequality, for every $s > 0$,
\begin{equation}
\label{eq:cutoff-markov}
\Pr_{a\sim r(p)}[|a - p| \ge s]\le \frac{10^{-6}\varepsilon^{2/3}}{s^2}.
\end{equation}

We choose $m = \lfloor \varepsilon^{-1/3}\rfloor$ evenly spaced points $p_1,\ldots,p_m\in [1/3,2/3]$, where each $p_i = 1/3 + i/(3m)$. We define the distribution $\cD$ of $(p,y)$ as follows. For every $i = 1,\ldots,m$, we set $\Pr[p = p_i] = 3\varepsilon$. Thus the total probability mass on $\{p_1,\ldots,p_m\}$ is $3m\varepsilon = \Theta(\varepsilon^{2/3}) < 1$. We place the remaining $1 - 3m\varepsilon$ probability mass on $p = 0$. The label $y$ is distributed as follows:
\[
\E[y|p] = \begin{cases}
    0, & \text{if } p = 0;\\
    p_i  + (-1)^i /3, & \text{if $p = p_i$ for $i\in \{1,\ldots,m\}$}.
\end{cases}
\]
It is easy to verify that $\cutoff(\cD) = \varepsilon$ because the bias $(-1)^i/3$ in every two neighboring bins cancel out, so the bias of exactly one bin with weight $3\varepsilon$ contributes to $\cutoff(\cD)$. It remains to show that $\sr_Z(r,\cD) = \Omega(\varepsilon^{2/3})$. We define
\[
\varphi(a):= \sum_{i = 1}^m (-1)^i\ind[|a - p_i| < 1/(6m)] \in \{-1,1\}.
\]
Since $p_1,\ldots,p_m\in [1/3,2/3]$, 
\begin{equation}
    \label{eq:zero}
    \varphi(a) = 0, \quad \text{for every $a\in [0,1/3 - 1/(6m)]$}
\end{equation}
Now we bound the quantities related to $\varphi$ in \eqref{eq:varphi}.
\begin{align}
\E_\Gamma[\varphi(a)^2] & \le \Pr[p\ne 0] + \Pr[p = 0]\E_{a\sim r(0)}[\varphi(a^2)]\notag \\
& \le 3m\varepsilon + (1 - 3m\varepsilon) \Pr_{a\sim r(0)}[a \ge 1/3 - 1/(6m)] \tag{by \eqref{eq:zero}}\\
& \le O(\varepsilon^{2/3}). \tag{by \eqref{eq:cutoff-markov}}
\end{align}
Also,
\[
\E[(y - a)\varphi(a)|p = 0] \ge -\Pr_{a\sim r(0)}[a \ge 1/3 - 1/(6m)] \ge - \varepsilon^{2/3}/100,\tag{by \eqref{eq:zero} and \eqref{eq:cutoff-markov}}
\]
and 
\begin{align*}
\E[(y - a)\varphi(a)|p = p_i]  & = \E[(p_i + (-1)^i/3 - a)\varphi(a)|p = p_i]\\
& = \E_{a\sim r(p_i)}[(p_i + (-1)^i/3 - a)\varphi(a)]\\
& \ge \Pr_{a\sim r(p_i)}[|a - p_i| < 1/(6m)](1/3 - 1/(6m)) - \Pr_{a\sim r(p_i)}[|a - p_i| \ge 1/(6m)]\\
& \ge  (9/10) (1/3 - 1/(6m)) - 1/10 \tag{by \eqref{eq:cutoff-markov}}\\
& \ge 1/10.
\end{align*}
Combining the last two inequalities,
\[
\E[(y - a)\varphi(a)] \ge -\varepsilon^{2/3}/100 + 3m\varepsilon/10 \ge \Omega(\varepsilon^{2/3}).
\]
Plugging our bounds on $\E[\varphi(a)^2]$ and $\E[(y - a)\varphi(a)]$ into \eqref{eq:varphi} proves that $\E[\sr_Z(r,\cD)] = \Omega(\varepsilon^{2/3})$.
\end{proof}

%% file: main.bbl
\newcommand{\etalchar}[1]{$^{#1}$}
\begin{thebibliography}{HLNRY23}

\bibitem[AIGT{\etalchar{+}}22]{metrics-cal}
Imanol Arrieta-Ibarra, Paman Gujral, Jonathan Tannen, Mark Tygert, and Cherie
  Xu.
\newblock Metrics of calibration for probabilistic predictions.
\newblock {\em Journal of Machine Learning Research}, 23(351):1--54, 2022.
\newblock URL: \url{http://jmlr.org/papers/v23/22-0658.html}.

\bibitem[BGHN23a]{utc}
Jaros\l{}aw B\l{}asiok, Parikshit Gopalan, Lunjia Hu, and Preetum Nakkiran.
\newblock A unifying theory of distance from calibration.
\newblock In {\em Proceedings of the 55th Annual ACM Symposium on Theory of
  Computing}, STOC 2023, page 1727–1740, New York, NY, USA, 2023. Association
  for Computing Machinery.
\newblock \href {https://doi.org/10.1145/3564246.3585182}
  {\path{doi:10.1145/3564246.3585182}}.

\bibitem[BGHN23b]{when}
Jaroslaw Blasiok, Parikshit Gopalan, Lunjia Hu, and Preetum Nakkiran.
\newblock When does optimizing a proper loss yield calibration?
\newblock In A.~Oh, T.~Naumann, A.~Globerson, K.~Saenko, M.~Hardt, and
  S.~Levine, editors, {\em Advances in Neural Information Processing Systems},
  volume~36, pages 72071--72095. Curran Associates, Inc., 2023.
\newblock URL:
  \url{https://proceedings.neurips.cc/paper_files/paper/2023/file/e4165c96702bac5f4962b70f3cf2f136-Paper-Conference.pdf}.

\bibitem[BN24]{smoothECE}
Jaroslaw Blasiok and Preetum Nakkiran.
\newblock Smooth {ECE}: Principled reliability diagrams via kernel smoothing.
\newblock In {\em The Twelfth International Conference on Learning
  Representations}, 2024.
\newblock URL: \url{https://openreview.net/forum?id=XwiA1nDahv}.

\bibitem[Daw82]{calibration}
A.~P. Dawid.
\newblock The well-calibrated bayesian.
\newblock {\em Journal of the American Statistical Association},
  77(379):605--610, 1982.
\newblock \href {https://doi.org/10.1080/01621459.1982.10477856}
  {\path{doi:10.1080/01621459.1982.10477856}}.

\bibitem[DDF{\etalchar{+}}25]{breaking-barrier}
Yuval Dagan, Constantinos Daskalakis, Maxwell Fishelson, Noah Golowich, Robert
  Kleinberg, and Princewill Okoroafor.
\newblock Breaking the {T\^{}(2/3)} barrier for sequential calibration.
\newblock In {\em Proceedings of the 57th Annual ACM Symposium on Theory of
  Computing}, STOC '25, page 2007–2018, New York, NY, USA, 2025. Association
  for Computing Machinery.
\newblock \href {https://doi.org/10.1145/3717823.3718178}
  {\path{doi:10.1145/3717823.3718178}}.

\bibitem[FGMS25]{fishelson2026highdimensional}
Maxwell Fishelson, Noah Golowich, Mehryar Mohri, and Jon Schneider.
\newblock High-dimensional calibration from swap regret.
\newblock In {\em The Thirty-ninth Annual Conference on Neural Information
  Processing Systems}, 2025.
\newblock URL: \url{https://openreview.net/forum?id=UVDihUz0iT}.

\bibitem[FV97]{foster1997calibrated}
Dean~P. Foster and Rakesh~V. Vohra.
\newblock Calibrated learning and correlated equilibrium.
\newblock {\em Games and Economic Behavior}, 21(589):40--55, 1997.

\bibitem[FV98]{asymptotic-calibration}
Dean~P. Foster and Rakesh~V. Vohra.
\newblock Asymptotic calibration.
\newblock {\em Biometrika}, 85(2):379--390, 06 1998.
\newblock \href {https://doi.org/10.1093/biomet/85.2.379}
  {\path{doi:10.1093/biomet/85.2.379}}.

\bibitem[GHK{\etalchar{+}}23]{loss-oi}
Parikshit Gopalan, Lunjia Hu, Michael~P. Kim, Omer Reingold, and Udi Wieder.
\newblock {Loss Minimization Through the Lens Of Outcome Indistinguishability}.
\newblock In Yael Tauman~Kalai, editor, {\em 14th Innovations in Theoretical
  Computer Science Conference (ITCS 2023)}, volume 251 of {\em Leibniz
  International Proceedings in Informatics (LIPIcs)}, pages 60:1--60:20,
  Dagstuhl, Germany, 2023. Schloss Dagstuhl -- Leibniz-Zentrum f{\"u}r
  Informatik.
\newblock URL:
  \url{https://drops.dagstuhl.de/entities/document/10.4230/LIPIcs.ITCS.2023.60},
  \href {https://doi.org/10.4230/LIPIcs.ITCS.2023.60}
  {\path{doi:10.4230/LIPIcs.ITCS.2023.60}}.

\bibitem[GHR24]{multiclass}
Parikshit Gopalan, Lunjia Hu, and Guy~N. Rothblum.
\newblock On computationally efficient multi-class calibration.
\newblock In Shipra Agrawal and Aaron Roth, editors, {\em Proceedings of Thirty
  Seventh Conference on Learning Theory}, volume 247 of {\em Proceedings of
  Machine Learning Research}, pages 1983--2026. PMLR, 30 Jun--03 Jul 2024.
\newblock URL: \url{https://proceedings.mlr.press/v247/gopalan24a.html}.

\bibitem[GJRR24]{oracle-omni}
Sumegha Garg, Christopher Jung, Omer Reingold, and Aaron Roth.
\newblock Oracle efficient online multicalibration and omniprediction.
\newblock In David~P. Woodruff, editor, {\em Proceedings of the 2024 {ACM-SIAM}
  Symposium on Discrete Algorithms, {SODA} 2024, Alexandria, VA, USA, January
  7-10, 2024}, pages 2725--2792. {SIAM}, 2024.
\newblock \href {https://doi.org/10.1137/1.9781611977912.98}
  {\path{doi:10.1137/1.9781611977912.98}}.

\bibitem[GKR{\etalchar{+}}22]{omni}
Parikshit Gopalan, Adam~Tauman Kalai, Omer Reingold, Vatsal Sharan, and Udi
  Wieder.
\newblock {Omnipredictors}.
\newblock In Mark Braverman, editor, {\em 13th Innovations in Theoretical
  Computer Science Conference (ITCS 2022)}, volume 215 of {\em Leibniz
  International Proceedings in Informatics (LIPIcs)}, pages 79:1--79:21,
  Dagstuhl, Germany, 2022. Schloss Dagstuhl -- Leibniz-Zentrum f{\"u}r
  Informatik.
\newblock URL:
  \url{https://drops.dagstuhl.de/entities/document/10.4230/LIPIcs.ITCS.2022.79},
  \href {https://doi.org/10.4230/LIPIcs.ITCS.2022.79}
  {\path{doi:10.4230/LIPIcs.ITCS.2022.79}}.

\bibitem[GKR23]{equiv}
Parikshit Gopalan, Michael Kim, and Omer Reingold.
\newblock Swap agnostic learning, or characterizing omniprediction via
  multicalibration.
\newblock In A.~Oh, T.~Naumann, A.~Globerson, K.~Saenko, M.~Hardt, and
  S.~Levine, editors, {\em Advances in Neural Information Processing Systems},
  volume~36, pages 39936--39956. Curran Associates, Inc., 2023.
\newblock URL:
  \url{https://proceedings.neurips.cc/paper_files/paper/2023/file/7d693203215325902ff9dbdd067a50ac-Paper-Conference.pdf}.

\bibitem[GOR{\etalchar{+}}24]{omni-regression}
Parikshit Gopalan, Princewill Okoroafor, Prasad Raghavendra, Abhishek Sherry,
  and Mihir Singhal.
\newblock Omnipredictors for regression and the approximate rank of convex
  functions.
\newblock In Shipra Agrawal and Aaron Roth, editors, {\em Proceedings of Thirty
  Seventh Conference on Learning Theory}, volume 247 of {\em Proceedings of
  Machine Learning Research}, pages 2027--2070. PMLR, 30 Jun--03 Jul 2024.
\newblock URL: \url{https://proceedings.mlr.press/v247/gopalan24b.html}.

\bibitem[GSTT25]{efficient-decision}
Parikshit Gopalan, Konstantinos Stavropoulos, Kunal Talwar, and Pranay Tankala.
\newblock Efficient calibration for decision making.
\newblock {\em arXiv preprint arXiv:2511.13699}, 2025.

\bibitem[GSTT26]{importance}
Parikshit Gopalan, Konstantinos Stavropoulos, Kunal Talwar, and Pranay Tankala.
\newblock The importance of being smoothly calibrated.
\newblock {\em arXiv preprint arXiv:2603.16015}, 2026.

\bibitem[HHW25]{truthful}
Jason Hartline, Lunjia Hu, and Yifan Wu.
\newblock A perfectly truthful calibration measure.
\newblock {\em arXiv preprint arXiv:2508.13100}, 2025.

\bibitem[HJKRR18]{mc}
Ursula Hebert-Johnson, Michael Kim, Omer Reingold, and Guy Rothblum.
\newblock Multicalibration: Calibration for the
  ({C}omputationally-identifiable) masses.
\newblock In Jennifer Dy and Andreas Krause, editors, {\em Proceedings of the
  35th International Conference on Machine Learning}, volume~80 of {\em
  Proceedings of Machine Learning Research}, pages 1939--1948. PMLR, 10--15 Jul
  2018.
\newblock URL: \url{https://proceedings.mlr.press/v80/hebert-johnson18a.html}.

\bibitem[HLNRY23]{constrained-omni}
Lunjia Hu, Inbal~Rachel Livni~Navon, Omer Reingold, and Chutong Yang.
\newblock Omnipredictors for constrained optimization.
\newblock In Andreas Krause, Emma Brunskill, Kyunghyun Cho, Barbara Engelhardt,
  Sivan Sabato, and Jonathan Scarlett, editors, {\em Proceedings of the 40th
  International Conference on Machine Learning}, volume 202 of {\em Proceedings
  of Machine Learning Research}, pages 13497--13527. PMLR, 23--29 Jul 2023.
\newblock URL: \url{https://proceedings.mlr.press/v202/hu23b.html}.

\bibitem[HQYZ24]{truthful-HQYZ}
Nika Haghtalab, Mingda Qiao, Kunhe Yang, and Eric Zhao.
\newblock Truthfulness of calibration measures.
\newblock In A.~Globerson, L.~Mackey, D.~Belgrave, A.~Fan, U.~Paquet,
  J.~Tomczak, and C.~Zhang, editors, {\em Advances in Neural Information
  Processing Systems}, volume~37, pages 117237--117290. Curran Associates,
  Inc., 2024.
\newblock URL:
  \url{https://proceedings.neurips.cc/paper_files/paper/2024/file/d4cbcae8cfc8aa3ae897a1296e4e0cac-Paper-Conference.pdf},
  \href {https://doi.org/10.52202/079017-3722}
  {\path{doi:10.52202/079017-3722}}.

\bibitem[HTY25]{sim-mim}
Lunjia Hu, Kevin Tian, and Chutong Yang.
\newblock Omnipredicting single-index models with multi-index models.
\newblock In {\em Proceedings of the 57th Annual ACM Symposium on Theory of
  Computing}, STOC '25, page 1762–1773, New York, NY, USA, 2025. Association
  for Computing Machinery.
\newblock \href {https://doi.org/10.1145/3717823.3718223}
  {\path{doi:10.1145/3717823.3718223}}.

\bibitem[HTY26]{hu2026simultaneous}
Lunjia Hu, Kevin Tian, and Chutong Yang.
\newblock Simultaneous blackwell approachability and applications to multiclass
  omniprediction.
\newblock {\em arXiv preprint arXiv:2602.17577}, 2026.

\bibitem[HW24]{cdl}
Lunjia Hu and Yifan Wu.
\newblock Predict to minimize swap regret for all payoff-bounded tasks.
\newblock In {\em 2024 IEEE 65th Annual Symposium on Foundations of Computer
  Science (FOCS)}, pages 244--263, 2024.
\newblock \href {https://doi.org/10.1109/FOCS61266.2024.00024}
  {\path{doi:10.1109/FOCS61266.2024.00024}}.

\bibitem[HWY25]{smooth-decision}
Jason Hartline, Yifan Wu, and Yunran Yang.
\newblock {Smooth Calibration and Decision Making}.
\newblock In Mark Bun, editor, {\em 6th Symposium on Foundations of Responsible
  Computing (FORC 2025)}, volume 329 of {\em Leibniz International Proceedings
  in Informatics (LIPIcs)}, pages 16:1--16:26, Dagstuhl, Germany, 2025. Schloss
  Dagstuhl -- Leibniz-Zentrum f{\"u}r Informatik.
\newblock URL:
  \url{https://drops.dagstuhl.de/entities/document/10.4230/LIPIcs.FORC.2025.16},
  \href {https://doi.org/10.4230/LIPIcs.FORC.2025.16}
  {\path{doi:10.4230/LIPIcs.FORC.2025.16}}.

\bibitem[KF08]{smooth}
Sham~M. Kakade and Dean~P. Foster.
\newblock Deterministic calibration and {Nash} equilibrium.
\newblock {\em Journal of Computer and System Sciences}, 74(1):115--130, 2008.
\newblock Learning Theory 2004.
\newblock URL:
  \url{https://www.sciencedirect.com/science/article/pii/S0022000007000633},
  \href {https://doi.org/10.1016/j.jcss.2007.04.017}
  {\path{doi:10.1016/j.jcss.2007.04.017}}.

\bibitem[KKG{\etalchar{+}}22]{universal}
Michael~P. Kim, Christoph Kern, Shafi Goldwasser, Frauke Kreuter, and Omer
  Reingold.
\newblock Universal adaptability: Target-independent inference that competes
  with propensity scoring.
\newblock {\em Proceedings of the National Academy of Sciences},
  119(4):e2108097119, 2022.
\newblock URL: \url{https://www.pnas.org/doi/abs/10.1073/pnas.2108097119},
  \href
  {https://arxiv.org/abs/https://www.pnas.org/doi/pdf/10.1073/pnas.2108097119}
  {\path{arXiv:https://www.pnas.org/doi/pdf/10.1073/pnas.2108097119}}, \href
  {https://doi.org/10.1073/pnas.2108097119}
  {\path{doi:10.1073/pnas.2108097119}}.

\bibitem[KLST23]{u-cal}
Bobby Kleinberg, Renato~Paes Leme, Jon Schneider, and Yifeng Teng.
\newblock U-calibration: Forecasting for an unknown agent.
\newblock In Gergely Neu and Lorenzo Rosasco, editors, {\em Proceedings of
  Thirty Sixth Conference on Learning Theory}, volume 195 of {\em Proceedings
  of Machine Learning Research}, pages 5143--5145. PMLR, 12--15 Jul 2023.
\newblock URL: \url{https://proceedings.mlr.press/v195/kleinberg23a.html}.

\bibitem[KP23]{performative}
Michael~P. Kim and Juan~C. Perdomo.
\newblock {Making Decisions Under Outcome Performativity}.
\newblock In Yael Tauman~Kalai, editor, {\em 14th Innovations in Theoretical
  Computer Science Conference (ITCS 2023)}, volume 251 of {\em Leibniz
  International Proceedings in Informatics (LIPIcs)}, pages 79:1--79:15,
  Dagstuhl, Germany, 2023. Schloss Dagstuhl -- Leibniz-Zentrum f{\"u}r
  Informatik.
\newblock URL:
  \url{https://drops.dagstuhl.de/entities/document/10.4230/LIPIcs.ITCS.2023.79},
  \href {https://doi.org/10.4230/LIPIcs.ITCS.2023.79}
  {\path{doi:10.4230/LIPIcs.ITCS.2023.79}}.

\bibitem[LRS25]{lrs}
Jiuyao Lu, Aaron Roth, and Mirah Shi.
\newblock Sample efficient omniprediction and downstream swap regret for
  non-linear losses.
\newblock In Nika Haghtalab and Ankur Moitra, editors, {\em Proceedings of
  Thirty Eighth Conference on Learning Theory}, volume 291 of {\em Proceedings
  of Machine Learning Research}, pages 3829--3878. PMLR, 30 Jun--04 Jul 2025.
\newblock URL: \url{https://proceedings.mlr.press/v291/lu25b.html}.

\bibitem[NRRX23]{high-dim}
Georgy Noarov, Ramya Ramalingam, Aaron Roth, and Stephan Xie.
\newblock High-dimensional unbiased prediction for sequential decision making.
\newblock In {\em OPT 2023: Optimization for Machine Learning}, 2023.
\newblock URL: \url{https://openreview.net/forum?id=P4j4l45NUq}.

\bibitem[OKK25]{optimal-omni}
Princewill Okoroafor, Robert Kleinberg, and Michael~P. Kim.
\newblock Near-optimal algorithms for omniprediction.
\newblock In {\em 2025 IEEE 66th Annual Symposium on Foundations of Computer
  Science (FOCS)}, pages 1595--1609, 2025.
\newblock \href {https://doi.org/10.1109/FOCS63196.2025.00084}
  {\path{doi:10.1109/FOCS63196.2025.00084}}.

\bibitem[Pen25]{peng2025high}
Binghui Peng.
\newblock High dimensional online calibration in polynomial time.
\newblock {\em arXiv preprint arXiv:2504.09096}, 2025.

\bibitem[QV21]{QV}
Mingda Qiao and Gregory Valiant.
\newblock Stronger calibration lower bounds via sidestepping.
\newblock In {\em Proceedings of the 53rd Annual ACM SIGACT Symposium on Theory
  of Computing}, STOC 2021, page 456–466, New York, NY, USA, 2021.
  Association for Computing Machinery.
\newblock \href {https://doi.org/10.1145/3406325.3451050}
  {\path{doi:10.1145/3406325.3451050}}.

\bibitem[QZ25]{truthful-QZ}
Mingda Qiao and Eric Zhao.
\newblock Truthfulness of decision-theoretic calibration measures.
\newblock In Nika Haghtalab and Ankur Moitra, editors, {\em Proceedings of
  Thirty Eighth Conference on Learning Theory}, volume 291 of {\em Proceedings
  of Machine Learning Research}, pages 4686--4739. PMLR, 30 Jun--04 Jul 2025.
\newblock URL: \url{https://proceedings.mlr.press/v291/qiao25a.html}.

\bibitem[RS24]{Roth-Shi}
Aaron Roth and Mirah Shi.
\newblock Forecasting for swap regret for all downstream agents.
\newblock In {\em Proceedings of the 25th ACM Conference on Economics and
  Computation}, EC '24, page 466–488, New York, NY, USA, 2024. Association
  for Computing Machinery.
\newblock \href {https://doi.org/10.1145/3670865.3673622}
  {\path{doi:10.1145/3670865.3673622}}.

\bibitem[RSB{\etalchar{+}}25]{test-action}
Raphael Rossellini, Jake~A. Soloff, Rina~Foygel Barber, Zhimei Ren, and Rebecca
  Willett.
\newblock Can a calibration metric be both testable and actionable?
\newblock In Nika Haghtalab and Ankur Moitra, editors, {\em Proceedings of
  Thirty Eighth Conference on Learning Theory}, volume 291 of {\em Proceedings
  of Machine Learning Research}, pages 4937--4972. PMLR, 30 Jun--04 Jul 2025.
\newblock URL: \url{https://proceedings.mlr.press/v291/rossellini25a.html}.

\end{thebibliography}
